\PassOptionsToPackage{normalem}{ulem}
\documentclass[sigconf,nonacm,balance=false]{acmart}

\usepackage{subcaption}

\usepackage{hyperref}
\usepackage{xcolor}
\usepackage[english]{babel}
\usepackage[utf8]{inputenc}
\usepackage{framed}
\usepackage{float}
\usepackage{booktabs}
\usepackage[most]{tcolorbox}
\usepackage{algorithm}
\usepackage{algpseudocode}
\usepackage{amsmath}
\usepackage{tikz}
\usepackage{ulem}
\usepackage{threeparttable}

\usepackage[T1]{fontenc}
\usepackage{listings}

\let\oldBbbk\Bbbk
\let\Bbbk\relax
\usepackage{amsmath,amssymb}
\let\Bbbk\oldBbbk

\definecolor{codegreen}{rgb}{0,0.6,0}
\definecolor{codegray}{rgb}{0.5,0.5,0.5}
\definecolor{codepurple}{rgb}{0.58,0,0.82}
\definecolor{backcolour}{rgb}{0.95,0.95,0.95}

\lstdefinelanguage{tamarin}{
  morekeywords={restriction, rule, lemma, let, in, if, then, else, All, ==>, -->, not, &, ...}, 
  sensitive=true,
  morecomment=[l]{//},
}

\lstdefinestyle{tamarinstyle}{
    backgroundcolor=\color{backcolour},   
    commentstyle=\color{codegreen},
    keywordstyle=\color{blue},
    numberstyle=\tiny\color{codegray},
    stringstyle=\color{codepurple},
    basicstyle=\ttfamily\footnotesize,
    breakatwhitespace=false,         
    breaklines=true,                 
    captionpos=b,                    
    keepspaces=true,                 
    numbers=left,                    
    numbersep=5pt,                  
    showspaces=false,                
    showstringspaces=false,
    showtabs=false,                  
    tabsize=2,
    frame=single,
    breakindent=2em,                 
    postbreak=\mbox{\textcolor{red}{$\hookrightarrow$}\space}, 
    morekeywords={theory, begin, end, builtins, rule, restriction, lemma, let, in, All, Ex, not, true, Fr, In, Out}
}

\newcommand{\revision}[1]{\textcolor{black}{#1}}
\newcommand{\revisiont}[1]{\textcolor{black}{#1}}

\tcbset{left=1mm,right=1mm,top=1mm,bottom=1mm}
\usepackage{breakcites}

\newcommand{\mnote}[1]{}
\newcommand{\zxnote}[1]{}
\newcommand{\james}[1]{}
\newcommand{\adria}[1]{}
\newcommand{\todo}[1]{}

\def\ddefloop#1{\ifx\ddefloop#1\else\ddef{#1}\expandafter\ddefloop\fi}

\def\ddef#1{\expandafter\def\csname bb#1\endcsname{\ensuremath{\mathbb{#1}}}}
\ddefloop ABCDEFGHIJKLMNOPQRSTUVWXYZ\ddefloop

\def\ddef#1{\expandafter\def\csname c#1\endcsname{\ensuremath{\mathcal{#1}}}}
\ddefloop ABCDEFGHIJKLMNOPQRSTUVWXYZ\ddefloop

\def\ddef#1{\expandafter\def\csname v#1\endcsname{\ensuremath{\boldsymbol{#1}}}}
\ddefloop ABCDEFGHIJKLMNOPQRSTUVWXYZabcdefghijklmnopqrstuvwxyz\ddefloop

\def\ddef#1{\expandafter\def\csname v#1\endcsname{\ensuremath{\boldsymbol{\csname #1\endcsname}}}}
\ddefloop {alpha}{beta}{gamma}{delta}{epsilon}{varepsilon}{zeta}{eta}{theta}{vartheta}{iota}{kappa}{lambda}{mu}{nu}{xi}{pi}{varpi}{rho}{varrho}{sigma}{varsigma}{tau}{upsilon}{phi}{varphi}{chi}{psi}{omega}{Gamma}{Delta}{Theta}{Lambda}{Xi}{Pi}{Sigma}{varSigma}{Upsilon}{Phi}{Psi}{Omega}{ell}\ddefloop

\newif\ifeprint

\begin{document}

\title{Securing Private Federated Learning in a Malicious Setting: A Scalable TEE-Based Approach with Client Auditing}


\author{Shun Takagi}
\authornote{Corresponding author.}
\affiliation{%
  \institution{LY Corporation}
  \country{Japan}}
\email{shutakag@lycorp.co.jp}

\author{Satoshi Hasegawa}
\affiliation{%
  \institution{LY Corporation}
  \country{Japan}}
\email{satoshi.hasegawa@lycorp.co.jp}

\renewcommand{\shortauthors}{Takagi et al.}

\begin{abstract}
In cross-device private federated learning, differentially private follow-the-regularized-leader (DP-FTRL) has emerged as a promising privacy-preserving method. 
However, existing approaches assume a semi-honest server and have not addressed the challenge of securely removing this assumption. 
This is due to its statefulness, which becomes particularly problematic in practical settings where clients can drop out or be corrupted.  
While trusted execution environments (TEEs) might seem like an obvious solution, a straightforward implementation can introduce forking attacks or availability issues due to state management. 
To address this problem, our paper introduces a novel server extension that acts as a trusted computing base (TCB) to realize maliciously secure DP-FTRL. 
The TCB is implemented with an ephemeral TEE module on the server side to produce verifiable proofs of server actions. 
Some clients, upon being selected, participate in auditing these proofs with small additional communication and computational demands. 
This extension solution reduces the size of the TCB while maintaining the system's scalability and liveness. 
We provide formal proofs based on interactive differential privacy, demonstrating privacy guarantee in malicious settings.
Finally, we experimentally show that our framework adds small constant overhead to clients in several realistic settings.
\end{abstract}


\maketitle

\setcounter{page}{1}
\section{Introduction}
\label{sec:intro}
Cross-device Private Federated Learning (PFL) has recently garnered significant attention as a privacy-preserving learning framework across distributed clients' data~\cite{xu-etal-2023-federated, ji2025private, line-federated}. 
In a typical cross-device PFL setting, a single central server (referred to simply as a server hereafter) aims to train a unified machine learning model by collecting model updates instead of raw data from numerous lightweight clients. 
To provide a provable guarantee that privacy leakage from model updates is bounded, a differential privacy (DP)~\cite{dwork2006our} mechanism injects noise into the (aggregated) updates.

The FL systems should be designed to uphold key privacy principles~\cite{bonawitz2021federated, house2012consumer}. 
DP supports this goal by enforcing data minimization and anonymization. 
However, many current DP FL methods lack mechanisms to ensure transparency, verifiability, and auditability from the clients' perspective~\cite{daly2024federated}. 
That is, they heavily rely on the assumption that the server behaves \textit{honestly}.
A dishonest server or an external attacker \revision{(i.e., an adversary)} might gain sensitive information more than clients expect, without their awareness, causing the DP guarantee to fail.

\revision{
    To address this gap, we adopt a \textit{malicious} adversarial model (i.e., the Dolev-Yao model~\cite{dolev1983security}) rather than the weaker semi-honest model. A \textit{semi-honest} adversary correctly follows the protocol specifications but attempts to glean sensitive information from the execution. A \textit{malicious} adversary, however, poses a greater threat as they are not bound by the protocol; they can deviate from it arbitrarily, which includes corrupting the server or a subset of clients.
}

Among the various DP approaches, methods based on differentially private stochastic gradient descent (DP-SGD)~\cite{abadi2016deep}, such as DP-FedAvg~\cite{mcmahan2018learning} and DP-SCAFFOLD~\cite{noble2022differentially}, have been widely used. 
Extensive research has focused on making these methods at least partially secure in settings with an untrusted server. 
Secure Multiparty Computation (SMPC)~\cite{bonawitz2017practical} can ensure the secure aggregation of updates and can be adapted for PFL, such as through distributed differential privacy~\cite{agarwal2021skellam, kairouz2021distributed}. 
However, these methods assume the honesty of clients, and corrupted clients—such as in Sybil attacks~\cite{douceur2002sybil}\footnote{\revision{In this paper, we define Sybil attacks as scenarios where an adversary controls multiple client identities (i.e., Sybils) by possessing their secret keys, which allows a portion of the clients to collude under the adversary's direction.}}—can reduce the noise added for privacy.
Moreover, achieving essential privacy amplification with subsampling~\cite{abadi2016deep} requires complex random device sampling protocols, which are challenging to securely implement in a distributed environment (e.g., relying on non-colluding servers~\cite{talwar2024samplable}). 
Another approach, shuffling~\cite{girgis2021shuffled}, involves client-side perturbation and anonymization~\cite{erlingsson2019amplification} of updates. 
Shuffling is also susceptible to Sybil attacks, compromising the privacy of remaining clients' updates if a large number of clients are corrupted~\cite{kairouz2021advances}. 
Additionally, shuffling often results in a worse utility-privacy trade-off or higher communication costs compared to other methods~\cite{balle2020private}.

That is, these approaches require subsampling~\cite{abadi2016deep} or shuffling~\cite{erlingsson2019amplification} of updates, which are difficult to implement effectively in malicious settings.
Even in semi-honest settings, these approaches are impractical in a cross-device environment since the server has limited control over the subset of training data it encounters at any time~\cite{kairouz2021advances}. 
In this context, DP-FTRL (Differentially Private Follow-the-Regularized-Leader)~\cite{kairouz2021practical} has emerged as an alternative to DP-SGD-based methods, eliminating the need for such techniques. 
Moreover, recent studies have demonstrated, both theoretically and empirically, that DP-FTRL may offer a better privacy-utility trade-off than DP-SGD~\cite{choquettecorrelated, choquette2024amplified}. 
Therefore, DP-FTRL appears promising for PFL in malicious settings.

However, implementing DP-FTRL in malicious environments introduces new challenges due to its stateful nature~\cite{daly2024federated}. 
Unlike DP-SGD, which aggregates updates with DP independently in each round, DP-FTRL requires maintaining state throughout the training process, leading to a more complex server implementation that is difficult to verify for clients. 
Specifically, DP-FTRL requires two steps for planning and aggregation that takes previous rounds into account. 
During the planning phase, the server selects a new set of participants, referred to as a cohort, based on the clients' participation history to manage privacy leakage by bounding contribution. 
In the aggregation phase, the server integrates the cohort's updates with noise that is correlated with noise used in earlier rounds, reducing the total noise required.

Recent work~\cite{ball2024secure, bienstock2024dmm} has attempted to bridge this gap by proposing SMPC for such noise, but this relies on the assumption that the planning phase honestly determines the majority cohort.
These approaches do not effectively protect against a malicious server that can freely select corrupted clients (i.e., Sybils) during the planning phase to conduct Sybil attacks. 
Notably, a malicious server could exploit the planning phase with Sybils to repeat an iteration with different cohorts to uncover previous noise from correlations, thus recovering raw past updates. 
Since clients cannot detect repeated iterations due to changes in cohorts, preventing such attacks is impossible without communication. 
Consequently, achieving DP-FTRL in a malicious setting remains a challenging issue.

\subsection*{Contributions}

\subsubsection*{Problem Formulation}
To address this challenge, we begin research on developing a maliciously secure system \revision{long-running FL with} DP-FTRL that supports adaptive\footnote{Adaptivity is crucial because a static client selection strategy risks losing large amounts of data due to client dropouts, which can severely degrade model quality. Furthermore, in applications where users adaptively gather data, those who have collected data should participate~\cite{ji2025private}.} client participation. 
We pose the following research question:
\begin{center}
\textit{What constitutes ideal security for a maliciously secure DP-FTRL system for \revision{long-running} FL involving lightweight devices that can drop out, and how can we achieve it at a low cost?}
\end{center}

To formalize this problem, we adopt interactive DP~\cite{vadhan2023concurrent} in a distributed manner. 
The $(\varepsilon,\delta)$-interactive DP framework assumes only predefined client algorithms, while the server protocol can be arbitrary. 
This ensures that any server protocol concurrently interacting with client algorithms (i.e., star topology) cannot infer distributed data beyond what is guaranteed by $(\varepsilon,\delta)$-DP. 
Therefore, this formalization accurately reflects our malicious setting.

As described above, DP-FTRL needs to manage its state throughout the execution process.
Therefore, we envision the set of client algorithms as a concurrent system that maintains a shared object with all clients, representing the state, such as participation history. 
In this framework, we view DP-FTRL as a series of sequential operations composed of two key steps on the concurrent system with consistent state: planning with \textit{consistent} participation history and secure aggregation with \textit{consistent} noise. 
Within the concurrent system, these two steps constitute a process initiated by a (potentially malicious) server, which updates the shared object.

Since breaking consistency can lead to privacy leaks as mentioned above, a necessary condition for the concurrent system to satisfy interactive DP is to maintain process consistency with respect to the shared object, ensuring that a malicious server cannot complete any process that compromises this consistency.
Thus, we focus on ensuring the consistency of the concurrent system. 
Specifically, the system must uphold:
\begin{itemize}
\item Integrity: The process loads the shared object and adhere to the intended protocol (i.e., secure aggregation with the correct clients and correct noise).
\item Linearizability~\cite{herlihy1990linearizability}: The history of concurrent processes is equivalent to a history of sequential processes with respect to the shared object (i.e., participation history and noise added in the past rounds).
\end{itemize}
Implementing a linearizable system in a malicious environment requires global consensus (involving all clients) at each iteration, which is prohibitively expensive for large-scale FL.
Without global consensus, fork linearizability is the strongest consistency that can be achieved~\cite{mazieres2002building}, but it is not sufficient for DP-FTRL because it still permits the aforementioned Sybil attack.
\revision{
Instead, we rely on a probabilistic linearizability guarantee, which is sufficient for our purpose.}

\subsubsection*{Proposed Method}
To address the aforementioned problem, we incorporate Trusted Execution Environments (TEEs), which are gaining attention as future solutions for secure infrastructure~\cite{eichner2024confidential, russinovich2023confidential, apple-private-cloud-compute} in private data analytics, including FL infrastructure\footnote{https://research.google/blog/advances-in-private-training-for-production-on-device-language-models/}~\cite{zhang2023private,daly2024federated}. 
\revision{
Although this paper focuses on Intel SGX~\cite{intel-sgx} as a concrete example of a TEE, our method is applicable to any TEE that supports attested execution~\cite{pass2017formal}.
}
While running the entire FL logic within a TEE could serve as a potential solution, functioning as a trusted third party, naive designs are susceptible to forking attacks on the TEE-managed state~\cite{wilde2024forking}, which breaks the consistency.

\begin{figure}
    \centering
    \includegraphics[width=0.75\linewidth]{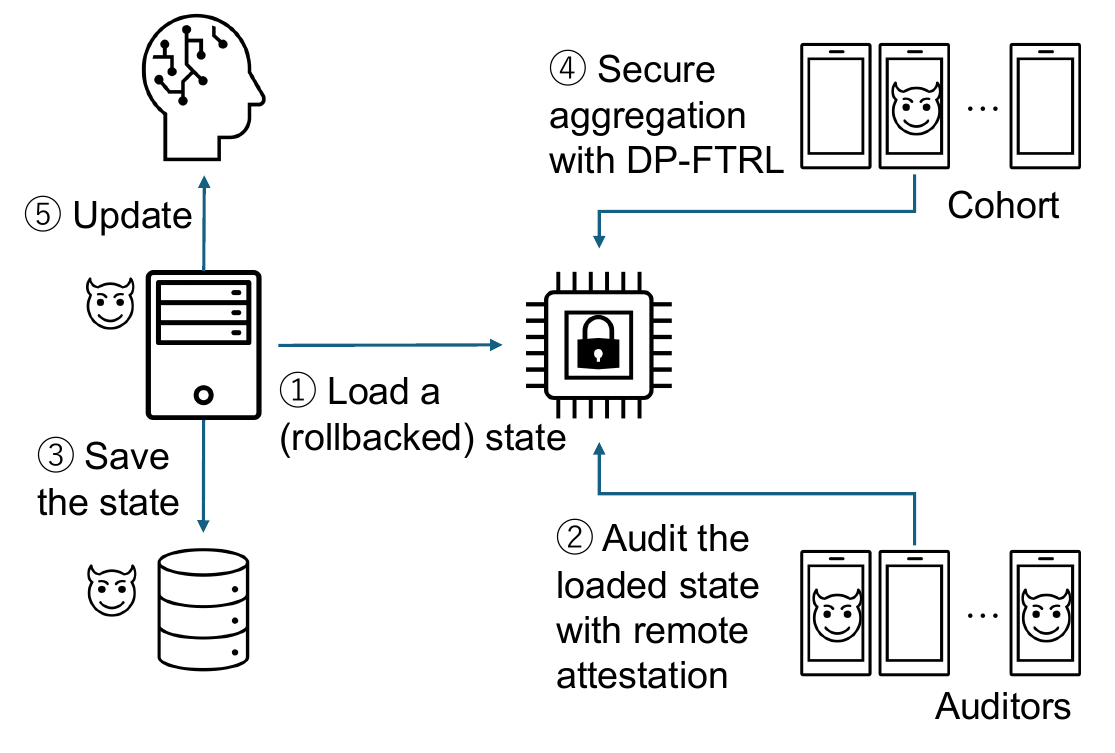}
    \caption{\revision{The process overview of each round of our system.}}
    \label{fig:intro-overview}
\end{figure}

\begin{table*}[]
  \centering
  \begin{threeparttable}
    \caption{\revision{Comparison of TEE-based systems for stateful applications. Our approach achieves consistency and liveness guarantees with a small TCB and without requiring external trusted hardware.}
    }
    \label{tab:comparison}
    \begin{tabular}{@{}lccccc@{}}
      \toprule
      \textbf{Approach} & \textbf{Consistency} & \textbf{Liveness} & \textbf{
      \revisiont{No External Trusted Hw.}
      } & \textbf{Formal Verification} & \textbf{TCB Size} \\
      \midrule
      TEEs without sealing & \checkmark & $\times$ & \checkmark & n/a & Small \\
      TEEs with sealing & $\times$ & \checkmark & \checkmark & n/a & Small \\
      SMR on TEEs without sealing~\cite{angel2023nimble,howard2023confidential,matetic2017rote} & \checkmark & $\times$\tnote{\dag} & \checkmark & n/a & Medium \\
      SMR on TEEs with sealing~\cite{wang2022engraft,niu2022narrator} & $\times$\tnote{\ddag} & \checkmark & \checkmark & n/a & Large \\
      TEEs + \revisiont{external} trusted hardware~\cite{strackx2016ariadne} 
      & \checkmark & \checkmark & $\times$ & \checkmark & Small \\
      \textbf{Ours} & \checkmark\tnote{*} & \checkmark\tnote{*} & \checkmark & \checkmark & \textbf{Small} \\
      \bottomrule
    \end{tabular}
    \begin{tablenotes}
      \item * Ours provides probabilistic guarantees; These properties hold with high probability, failing only with a small chance (e.g., $10^{-8}$).
      \item \dag\ Liveness is achieved under the assumption that a majority of TEEs do not experience faults (i.e., disaster).
      \item \ddag\ Consistency is achieved under the assumption that a majority of TEEs remain uncompromised by adversaries.
      \item n/a: Not addressed or not available in the literature.
    \end{tablenotes}
  \end{threeparttable}
\end{table*}

To tackle this issue, we propose a novel system that integrates TEEs with client auditing.
\revision{Figure~\ref{fig:intro-overview} illustrates the overview of our system during each update.}
\revision{
(1) For each round, the server feeds the current system state, along with the arguments for secure aggregation, into the TEE.}
(2) Upon successful client auditing of the state within the TEE with remote attestation, the TEE generates proofs of integrity and linearizability. 
(3) These proofs, bundled with the newly validated state, are then saved to untrusted storage.
(4) Then, the secure aggregation process within the TEE with the input arguments is completed only after clients have successfully verified it. 
\revision{Should an adversary attempt a fork or rollback attack by inputting a fraudulent state, the client-side verification would fail, thus preventing the secure aggregation from executing.}

\revision{
However, guaranteeing linearizability poses a significant hurdle, as it would still necessitate consensus among all clients. 
A malicious server could otherwise exploit any set of corrupted auditors to approve a proof that violates consistency. 
Our system mitigates this vulnerability by employing randomized auditing. 
In each round, a random subset of clients is tasked with auditing the server's proof, which makes it infeasible for an adversary to ensure its colluders are selected. 
As a result, the probability of a forged proof being successfully validated becomes small (e.g., $10^{-8}$). 
Consequently, our protocol forgoes a linearizability guarantee in favor of a probabilistic one, with a failure rate small enough for practical applications.
}

This system acts as an extension to the existing FL system, where the server merely appends proof to an original message, which each client verifies. 
Moreover, the additional computation and communication costs for the client are constant and small, enabling the system to preserve the scalability and liveness of the original arrangement.

\subsubsection*{Related Work}
\label{sec:intro-related-work}
\revisiont{
While TEEs provide a powerful primitive for attested execution—conceptually abstracted as $\mathcal{G}^\text{rollback}_\text{att}$~\cite{bhatotia2021steel}, they are not known to offer a self-contained solution for preventing rollback attacks. 
One strategy to this problem avoids TEE sealing capabilities, reducing the TEE to a stateless primitive (i.e., $\mathcal{G}_\text{att}$~\cite{pass2017formal}) that ensures consistency at the cost of liveness. Another major strategy augments the TEE with an external source of trust, such as an external hardware monotonic counter like Trusted Platform Module (TPMs).
}
In Table~\ref{tab:comparison}, we compare our proposed system against these existing approaches based on five key properties important for secure, long-running FL with DP-FTRL on TEEs in a malicious setting.

\revision{
    First, the most straightforward TEE-based implementations either forgo state persistence or use native sealing. 
    An implementation without sealing maintains consistency and a small trusted computing base (TCB) (i.e., the small size of code in the TEE, which reduces the need to implicitly trust its correctness) but fails to provide liveness; the state is irrevocably lost if the TEE instance crashes. 
    Conversely, an implementation with sealing ensures liveness by persisting the state to untrusted memory, but it becomes vulnerable to forking and rollback attacks~\cite{wilde2024forking}, thus failing to guarantee consistency.
}

\revision{
    Recent work has sought to address both consistency and liveness by implementing State Machine Replication (SMR) across multiple TEEs. Systems like Engraft~\cite{wang2022engraft} and Narrator~\cite{niu2022narrator} achieve consistency, but their guarantees rely on an honest majority of TEE replicas. This assumption does not hold in our threat model, where an adversary may control (e.g., rollback) the majority of nodes.
}

\revision{
    Nimble~\cite{angel2023nimble}, CCF~\cite{howard2023confidential}, and ROTE~\cite{matetic2017rote} operate under a threat model similar to ours. They provide consistency even in the presence of a malicious majority, not by relying solely on sealing APIs for node recovery, but by leveraging consensus among the live nodes.
    However, this design choice means they cannot guarantee liveness in the face of catastrophic failures (e.g., a crash of the majority of TEEs). 
    Furthermore, these systems inherit the complexity of SMR protocols inside TEEs, leading to a large TCB, which is undesirable according to the principles of TEEs~\cite{arnautov2016scone}. 
    As a result, formally verifying their security is challenging\footnote{Nimble~\cite{angel2023nimble} made considerable efforts in this area but did not fully verify functions such as reconfiguration. https://github.com/microsoft/Nimble/issues/6}, and complex in-enclave implementations are more susceptible to bugs. Indeed, vulnerabilities related to consistency have been discovered in some of these systems~\cite{wang2022engraft, niu2022narrator}.
    }
    
\revision{
    Another line of work combines TEEs with 
    \revisiont{additional trusted hardware components (e.g., TPMs) external to the TEE to enforce state integrity~\cite{parno2011memoir, strackx2016ariadne}.}
    These solutions are often simple, formally verifiable, and feature a small TCB. 
    However, they depend on the availability of specialized hardware or a trusted monotonic counter, and securely integrating it with a TEE can be challenging. 
    For example, the SGX-native monotonic counter is now deprecated, and combining SGX with an external TPM is susceptible to man-in-the-middle attacks like the Cuckoo attack~\cite{parno2008bootstrapping}, where a TEE is tricked into communicating with a remote, adversary-controlled TPM instead of the local one, allowing forking attacks.
    The vulnerability arises because, as explicitly described in the paper~\cite{strackx2016ariadne}, the method requires the TEE to be locally bound (i.e., on the same machine) to the TPM to prevent forking attacks. 
    However, an enclave cannot verify this local binding due to the threat of the Cuckoo attack. 
    A proposed countermeasure, verifying adjacency by measuring communication timing as proposed by Fink et al.~\cite{fink2011catching}, is ineffective in our setting, which may involve high-speed networks like a LAN. 
    Therefore, establishing a truly secure and trusted channel between a TEE and a TPM has another challenging problem.
}

\revision{
Conceptually, our method for achieving consistency is analogous to TPM-based solutions. However, instead of relying on  
\revisiont{external trusted hardware,}
our system leverages client auditing. This approach inherits the benefits of a simple and small TCB, enabling formal verification. It provides recovery from fault (i.e., liveness) and requires only a single TEE-enabled server, thereby reducing management overhead. 
More general related work is provided in Appendix~\ref{sec:related work}.
}

\subsubsection*{Contributions}

In summary, our proposed system realizes DP-FTRL with the following desired characteristics:
\begin{itemize}
    \item Privacy: Our approach preserves interactive DP (i.e., DP under the malicious setting). \revision{Smaller} reliance on TEE (i.e., small TCB size) reduces security risks and enhances verifiability.
    \item Availability: Our system retains its standard resilience to dropouts—training can continue even if some devices become nonresponsive, and it is possible to recover from crashes (i.e., liveness).
    A malicious client cannot halt the system via fraudulent audits.
    \item Scalability: Clients perform only constant lightweight auditing, incurring small computation and communication overhead. The system operates with large-scale clients that can drop out and progresses multiple processes concurrently.
\end{itemize}

\subsubsection*{\revision{Outiline}}
\label{sec:outline}
\revision{The remainder of this paper is structured as follows.
In Section~\ref{sec:preliminaries}, we provide the necessary background on key concepts, including TEEs, DP-FTRL, and the security models we employ.
We then formalize our problem in Section~\ref{sec:problem_statement}, defining the requirements for a maliciously secure DP-FTRL system in terms of integrity and linearizability, and introduce our simulation-based proof framework.
Section~\ref{sec:design_planner} details the core design of our proposed system. 
We provide a thorough analysis of our system in Section~\ref{sec:analysis}, covering its privacy guarantees, liveness under client dropouts, scalability, and guidance on parameter selection.
In Section~\ref{sec:experiments}, we present experimental results that demonstrate the practical performance of our system, focusing on the trade-offs between communication cost, security, and liveness, as well as the concrete overhead of our implementation.
We discuss side-channel attacks and implementation challenges in Appendix~\ref{sec:discussion} and situate our work within the existing literature in Appendix~\ref{sec:related work}.
}

\section{Preliminaries}\label{sec:preliminaries}

In this section, we review some key concepts and frameworks relevant to our proposed method, including Trusted Execution Environments (TEEs), Differentially Private Follow-the-Regularized-Leader (DP-FTRL), the concurrent system model, and security notions.

\subsection{Trusted Execution Environments (TEEs)}
\label{sec:tee}
\revision{
Although this paper focuses on Intel SGX~\cite{intel-sgx} as a concrete example of a TEE, our method is applicable to any TEE that provides attested execution~\cite{pass2017formal, bhatotia2021steel} like AMD SEV-SNP~\cite{amd-sev} (i.e., confidentiality and integrity of execution).
}
Below, we offer a high-level overview and refer the reader to~\cite{costan2016intel} for more intricate details.

\subsubsection{Enclaves}
\label{sec:enclave}
TEEs utilize trusted hardware to provide isolated environments known as enclaves. 
Enclaves protect memory by restricting access, thereby ensuring data confidentiality from unauthorized processes, including the operating system. 
Enclaves can produce an attestation report, which we call evidence. 
This report typically includes a hash of the running code (to verify its integrity), additional data related to the enclave's execution, and a signature using a key endorsed by the TEE manufacturer.
In this paper, an enclave can be seen as an entity with functions that generate evidence. Due to confidentiality guarantees, only the final evidence is visible. 
Moreover, enclaves can generate a secure nonce valid for the duration of the function's execution to thwart replay attacks.

\paragraph{Sealing}
Enclaves are volatile but can be made persistent through sealing. 
Specifically, they can securely encrypt and store states outside the enclave, allowing replication by an enclave with the same binary.

\subsubsection{Remote Attestation}
\label{sec:remote_attestation}
The signature in the evidence allows a remote party to verify that the evidence was generated by a genuine enclave running on authentic TEE hardware.
A client can authenticate this evidence through various methods, including certificate chains or contacting an attestation service. 
The included hash of the code and computation results enable clients to affirm integrity.

\subsubsection{Secure Aggregation}
\label{sec:secure_aggregation}
TEEs can perform secure aggregation of distributed data $x_i$~\cite{chamani2020mitigating,huba2022papaya,karl2021cryptonite}, enabling computation of $\sum_{i=1}^n x_i$ without exposing individual values $x_i$ to a potentially malicious server. 
Compared to SMPC~\cite{bonawitz2017practical}, which provides similar functionality, TEE-based secure aggregation requires smaller client-side operations, maintains constant complexity, and offers the flexibility to add complex noise without Sybil attacks.
The core idea is using an enclave as a trusted computing entity through remote attestation. Each client sets up a secure channel to the enclave using the Diffie-Hellman key exchange protocol. 
Client $i$ generates a symmetric key to encrypt $x_i$ and then transmits this symmetric key securely to the enclave. 
Since loading large data (i.e., encrypted model parameters) into an enclave can be time-intensive, methods such as stream ciphers~\cite{karl2021cryptonite} like homomorphic one-time pad~\cite{huba2022papaya} can enhance efficiency. 
See Section B.3 of~\cite{huba2022papaya} for a detailed methodology.

\subsection{Threat Model}
\label{sec:threat_model}
We consider a system comprising $n+1$ parties: one server and $n$ clients. 
The server has significant computational resources.
Each client holds private input data $x$ and a pair of private and public keys.
Clients can only communicate with the server, reflecting a star network topology. 
These clients are lightweight and can drop out, meaning they may have limitations in connectivity, bandwidth, and computational power.

\revision{
    An adversary aims to compromise the privacy of clients' sensitive values. We adopt the Dolev-Yao model~\cite{dolev1983security} as our model of the \textit{malicious} adversary. That is, an adversary can eavesdrop on all communication, inject messages, and corrupt parties by obtaining their private keys. However, the adversary cannot break cryptographic primitives, in this paper, digital signatures, nonces, Diffie-Hellman key exchanges, and the symmetric encryption.
}

In our threat model, we include a fixed server and clients listed in a public key list, assuming the following conditions with $\gamma,\kappa,\beta\in[0,1]$:
\begin{itemize}
    \item The adversary can corrupt the server and a fraction $\gamma$ of \textbf{all clients}. The compromised clients are fixed.
    \item At any given time, at least a fraction $\kappa$ of \textbf{all clients} are available, with this availability being variable.
    \item A fraction $\beta$ of \textbf{all available clients} may drop out, failing to fulfill their roles in the protocol.
\end{itemize}

This model \revision{specializes} the $(\gamma, \beta)$ secure federated MPC paradigm~\cite{ball2024secure}, tailoring it to more realistically account for the fraction of all clients rather than just any single cohort. 
No party, apart from an adversary, can distinguish corrupted clients. No party can discern which clients will drop out.

We presuppose a trusted Public Key Infrastructure (PKI). 
While clients do not communicate with each other directly, they leverage the PKI to validate the public key list. 
Specifically, a client sends a hash of a client list to the PKI, and the PKI certifies whether the list is valid. A valid list is one where the proportion of malicious clients is at most $\gamma$ and the total number of clients is above a certain threshold. 
This circumscribed role for the PKI reduces our dependency on it compared to other SMPC frameworks~\cite{ball2024secure, bonawitz2017practical}.

We assume that enclaves provide confidentiality and integrity as described in Section~\ref{sec:enclave}. 
Thus, vulnerabilities against TEEs, such as side-channel attacks, are out of scope; however, \revision{a detailed discussion about side-channel attacks on our system is provided in Appendix~\ref{sec:side-channel-attacks}.}

\subsection{Concurrent System}
\label{sec:concurrent_system}

In the threat model outlined above, we consider a distributed concurrent system of clients working cooperatively with a server, all operating on a single object with one type of operation. 
Generally, an operation is denoted as \texttt{<q op(args) A>}\texttt{<OK(response) A>}, where \texttt{A} is the process name, \texttt{op} is the type of operation, \texttt{args} is the list of arguments, and $q$ is the object. 
\texttt{OK} indicates the completion of process \texttt{A} and outputs the response.

For example, let $q$ be an ordered list, and consider an operation that appends an element $a$ to this list, outputting the result of $f(q, \text{args})$.
This could be denoted as \texttt{<q append(a, args) A>}\texttt{<OK(f(q, \text{args})) A>} with some given function $f$ and arguments \texttt{args}. 
In the \textit{concurrent} system, processes are invoked concurrently, meaning that process B can complete before process A, even if process A was invoked prior to process B.

\subsubsection{Integrity and Linearizability}\label{sec:integrity_linearizability}
The properties of the concurrent system that we focus on are \textbf{integrity} and \textbf{linearizability}.

\paragraph{Integrity of execution}
\revision{
    Integrity is the property that each individual operation behaves according to its specification. 
    For the above example, given a object $q$, element $a$, and arguments \texttt{args} for a function $f$, a system with integrity correctly appends $a$ to $q$ and computes and returns the result of $f(q, \texttt{args})$. This implies that the object's state, the arguments, and the operation itself are not tampered with, and the computation is performed as intended.
}

\paragraph{Linearizability~\cite{herlihy1990linearizability}}
\label{sec:linearizability}
\revision{
    Linearizability is a correctness condition for concurrent objects which ensures that each process appears to take effect instantaneously at some point between its invocation and response.
}
If a concurrent system respects linearizability, it ensures that the result of concurrent processes is equivalent to a result of sequential processes that are produced by the system that satisfies integrity and respects real-time precedence ordering. See~\cite{herlihy1990linearizability} for the more formal definition.
For instance, consider $f$ as a function that outputs $q$:

\begin{itemize}
    \item If a system produces \texttt{<q append(a\_1) A> <q append(a\_2) B><OK([a\_1, a\_2]) B> <OK([a\_1]) A>}, it respects linearizability; this is consistent with \texttt{<q append(a\_1) A><OK([a\_1]) A><q append(a\_2) B><OK([a\_1, a\_2]) B>}.
    \item However, a system that produces \texttt{<q append(a\_1) A><q append(a\_2) B><q OK([a\_2]) B><q OK([a\_1]) A>} does not, because the output of process B does not account for the update made by A, and thus it is not equivalent to a result of any sequential processes.
    \item \revision{
    Also, a system that produces \texttt{<q append(a2) B> <OK([a1, a2]) B> <q append(a1) A><q OK([a\_1]) A>} does not respect linearizability. Here, B completes before A invokes, establishing a real-time order where B precedes A. However, B's response depends on A's input, requiring A to precede B in any sequential explanation. This contradiction violates the real-time precedence condition.
    }
\end{itemize}

Unfortunately, under the threat model described in Section~\ref{sec:threat_model}, achieving linearizability is infeasible without the agreement of all clients in each process. The strongest consistency notion achievable under these conditions is captured by fork-linearizability~\cite{mazieres2002building, li2004secure}.
\revision{However, this is insufficient for DP-FTRL as it remains vulnerable to the Sybil attack mentioned in Introduction.
Instead, we rely on a probabilistic linearizability guarantee, which is sufficient for our purpose.
}

\subsection{Interactive Differential Privacy (Interactive DP)}
\label{sec:interactive_dp}

DP provides a privacy framework for randomized mechanisms that release outputs~\cite{dwork2014algorithmic, vadhan2017complexity}. 
It measures how close the probability distributions of a mechanism's outputs are when applied to two adjacent datasets. 
\revision{
The definition of adjacency is crucial. 
While it can be broadly defined by the addition/removal of a single user's record, in the context of FL where the participants are fixed, we adopt a more specific definition known as \textit{zero-out adjacency}~\cite{kairouz2021practical, choquette2024amplified}.
In this setting, two datasets $D$ and $D^\prime$, which are multisets over a data universe $\cX$, are called adjacent if they differ in the \textit{participation} of a single client. 
Here, the \textit{participation} means containing the client's true contribution, while in the other, that contribution is replaced with a vector of zeros to represent the client's non-participation.
}

\begin{definition}[Differential Privacy~\cite{dwork2006our}]
For $\varepsilon, \delta \ge 0$, a randomized mechanism $M : \cX^n \rightarrow \cZ$ is $(\varepsilon, \delta)$-differentially private if for every pair of adjacent datasets $D, D^\prime \in \cX^n$ and all subsets $T \subseteq \cZ$, the following holds: 
$$
\Pr[M(D) \in T] \le e^\varepsilon \cdot \Pr[M(D^\prime) \in T] + \delta
$$
where the randomness is over the internal coin flips of the algorithm $M$. 
\end{definition}

Recently, the concept of interactive DP~\cite{vadhan2021concurrent} was developed to handle scenarios where multiple DP mechanisms operate concurrently. 
Originally, this was explored in a centralized DP context, where a single party holds all sensitive data, to examine properties like the concurrent composition theorem~\cite{vadhan2023concurrent}.
In our setting, we consider the distributed setting of DP, where each client manages their own data~\cite{beimel2008distributed}. 
Note that we do not require each mechanism to meet the criteria of local DP~\cite{kasiviswanathan2011can}, instead focusing on interactive DP.

Here we present interactive DP applicable to our context. 
We focus on an interactive protocol $S$ for the concurrent composition of client mechanisms $M_1, M_2, \dots, M_n$, where $M_i$ has sensitive input $x_i$ for $i \in [n]$, denoted by $M = \textsc{ConComp}(M_1, M_2, \dots, M_n)$~\cite{vadhan2023concurrent}. 
Within this interactive protocol, $S$ can send messages to any client (i.e., $i \in [n]$) and receive responses from $M_i$. 
Each party can store received messages for crafting future responses. 
The view of the interactive protocol, $\texttt{View}\langle S, M(D)\rangle$, is defined as the sequence of messages that $S$ receives from $M_1, M_2, \dots, M_n$, where $D = (x_1, x_2, \dots, x_n) \in \cX^n$. 
For a more formal definition, see~\cite{vadhan2021concurrent}.

Interactive DP is defined as follows:
\begin{definition}[Interactive Differential Privacy~\cite{vadhan2021concurrent}]
A randomized algorithm $M$ is an $(\varepsilon, \delta)$-differentially private interactive mechanism if, for every pair of adjacent datasets $D, D^\prime \in \cX^n$, for every adversary algorithm $A$, and for every possible output set $T \subseteq \mathrm{Range} \left( \texttt{View} \langle A, M(\cdot)\rangle \right)$, the following holds:
$$
\Pr \left[ \texttt{View} \langle A, M(D)\rangle \in T \right] \le e^\varepsilon \Pr \left[ \texttt{View} \langle A, M(D^\prime)\rangle \in T \right] + \delta
$$
where the randomness is over the internal coin flips of both the algorithm $M$ and the adversary $A$. 
\end{definition}

\subsubsection*{\revision{Simulation Based Proof}}\label{sec:simulation_based_proof}
\revision{
Proving interactive DP is not straightforward because we cannot assume a semi-honest server-side protocol (e.g., the server can invoke processes concurrently). 
To address this, we adopt an approach similar
to Braun et al.~\cite{braun2024malicious} and Ball et al.~\cite{ball2024secure}, using a simulation-based method~\cite{lindell2017simulate} to prove the
interactive DP of $\textsc{ConComp}(M_1, M_2, \dots, M_n)$. 
This proof technique compares our real-world protocol to an idealized version.
The proof framework consists of two main steps: defining the ideal functionality and proving the emulation of the ideal functionality by the real-world protocol.
}

\revision{
First, we define an \textit{ideal functionality} that is executed by a trusted party in a hypothetical ideal world. 
This functionality specifies the perfect, secure execution of DP-FTRL. 
It interacts with a \textit{Simulator}, which represents the adversary in this ideal world. 
By design, the view of the simulator in this interaction is guaranteed to be DP. 
The ideal functionality is designed to allow the simulator to influence the protocol on behalf of corrupted parties, for example, by providing malicious inputs or choosing the order of round completion, thus modeling the power of a real-world adversary.
}

\revision{
Second, we demonstrate that our proposed concurrent system, $\textsc{ConComp}(M_1, M_2, \dots, M_n)$, \textit{emulates} this ideal functionality when interacting with any real-world adversary $A$. 
This is shown by constructing a simulator, Sim, that can produce a view computationally indistinguishable from the real adversary's (i.e., $A$) view, using only its interaction with the ideal functionality.
This ensures that any view created by $A$ matches the DP-compliant view in the ideal functionality, thereby ensuring compliance with interactive DP.
}

\subsection{Differentially Private Follow-the-Regularized-Leader (DP-FTRL)}
\label{sec:dp-ftrl}
Kairouz et al.~\cite{kairouz2021practical} proposed DP-FTRL as a DP optimization mechanism that avoids the need for subsampling~\cite{abadi2016deep} and shuffling~\cite{erlingsson2019amplification} which are difficult due to client dropouts. 
Subsequent research expanded DP-FTRL with the matrix mechanism~\cite{denisov2022improved}, called MF-DP-FTRL. 
This approach enables useful techniques such as multi-epoch training, amplification, and runtime efficiency due to its flexibility~\cite{choquette2024amplified, choquette2022multi, mcmahan2024hassle}. 
Recent studies indicate that MF-DP-FTRL can generalize DP-SGD and demonstrate its superiority in terms of the utility-privacy trade-off~\cite{choquettecorrelated, choquette2024amplified}.

We employ MF-DP-FTRL in our research to align our method with those operating within this framework. 
The pseudocode for the MF-DP-FTRL framework is located between Line $6$ and Line $17$ (where the simulator assigns $k$ as $i$, $C_\text{cohort}$ as valid client indices, and $\tilde{\theta}$ as $\tilde{\theta}_i$) in Algorithm~\ref{alg:ideal}. 
While resembling DP-SGD with gradient computation, clipping, and gradient-noise summation for updates, it deviates in the following key ways:

\begin{enumerate}
    \item Participation schema: client indices are arbitrarily chosen from those meeting the participation schema based on historical participation with an option for sampling.
    \item Stateful aggregation: noise is correlated with noise used in previous rounds.
\end{enumerate}
Due to these differences, unlike DP-SGD, each iteration in DP-FTRL cannot be treated as an independent DP mechanism. 
Thus, the server must handle participation history and noise management as detailed below.

\paragraph{Participation Schema~\cite{choquette2024amplified}}
Participation history $ H $ is a sequence of sets, each representing a client's participation record. 
Specifically, $ H_i $ is a set representing client $ i $'s participation pattern $ \pi\subseteq[n_{\text{round}}] $, where $n_\text{round}$ is the number of total rounds.
For instance, $H=(\{0,2\},\{1,3\})$ indicates that client $0$ participates in the 0th and 2nd rounds, while client $1$ participates in the 1st and 3rd rounds.
The participation schema is defined as follows:

\begin{definition}
A participation schema $\Pi$ is defined as the set of possible participation patterns $\pi \subseteq [n_{\text{round}}]$.
A participation history $ H $ adheres to a participation schema $\Pi$ if for all $ i \in [n] $, there exists $\pi \in \Pi$ such that $ H_i \subseteq \pi $.
\end{definition}

Updating $H$ with $C \subseteq [n]$ and index $i$ involves adding $i$ to $H_j$ for each $j \in C$. 
In this context, the current participation history $H$ is used to determine which clients are eligible to participate in the $i$th round by using the function $f_{{\rm qualify}}(\Pi, H, i)$, defined below. 
The function $f_{{\rm qualify}}(\Pi, H, i)$ returns the set $C_\text{qualify} \subseteq [n]$ such that even if $H$ is updated with $C_\text{qualify}$ and $i$, $H$ continues to adhere to the participation schema $\Pi$.

For example, if a client's participation is limited to just once, $\Pi = \{\{1\}, \{2\}, \ldots, \{n_{\text{iter}}\}\}$. 
Recent works have demonstrated that the $\Pi_b$ schema is both practical and attains utility comparable to DP-SGD by permitting multi-epoch training and amplification~\cite{choquette2024amplified}. 
This schema permits participation at intervals of $ b $ between submissions.

\paragraph{Stateful Aggregation}

In stateful aggregation, a noise matrix $ \mathbf{Z} \in \mathbb{R}^{n_{\text{round} \times d} }$ is employed throughout the DP-FTRL execution, with its elements independently drawn from a Gaussian distribution.\footnote{Not all elements of $ \mathbf{Z} $ need to be sampled initially; they can be incrementally sampled as required.}
The noise added in the $ i $-th execution relies on $ \mathbf{Z}_{jk} $ for $ j \in [i] $ and $ k \in [d] $. 
This method, which depends on noise from prior rounds, is termed correlated noise. 
The particular approach for calculating this noise is given by $ \mathbf{C} $.
Specifically, the noise for $i$th round is $\zeta(\mathbf{C}^{-1} \mathbf{Z})[i,:]$ where $\zeta$ is the clipping bound.

This method of noise integration is called the matrix mechanism in the adaptive streaming~\cite{denisov2022improved}, and in the context of DP-FTRL, $ \mathbf{C} $ is optimized according to a participation schema. 
Intuitively, this optimization minimizes the introduced error by controlling the sensitivity of participating clients based on the schema.
Refer to~\cite{choquette2024amplified, denisov2022improved} for more details.
Therefore, if clients deviate from the participation schema, their sensitivity may surpass allowable thresholds, risking privacy breaches.

\section{Problem Statement}
\label{sec:problem_statement}
This section aims to formalize the problem of implementing DP-FTRL in our malicious setting.

\subsection{DP-FTRL in the Malicious Setting}
\label{sec:dp-ftrl-in-the-malicious-setting}
In our threat model (see Section~\ref{sec:threat_model}), the server interacts with client algorithms $M_1, M_2, \dots, M_n$ to update the model. 
The server is malicious, so it can deviate from a predefined protocol, making it impractical to presume any specific server behavior.
This situation corresponds to concurrent composition in interactive DP (see Section~\ref{sec:interactive_dp}). 
Therefore, our goal is to construct a concurrent composition mechanism $M=\textsc{ConComp}(M_1, M_2, \dots, M_n)$ that can execute DP-FTRL with an appropriate server protocol, while also ensuring that $M$ maintains interactive DP to provide robustness against a malicious server.

\subsubsection{\revision{Strawman Approaches}}
\label{sec:strawman}
\revision{
It might seem that a straightforward implementation of DP-FTRL within a TEE could serve as a solution, as the TEE can act as a trusted third party to correctly manage the state.
However, naive implementations fail to provide both liveness and the interactive DP guarantee. 
We illustrate this by presenting two strawman approaches.
}

\subsubsection*{With Sealing}
\revision{
Since an enclave is volatile, its state is lost when it crashes. 
To prevent this, the server needs to \textit{seal} the enclave's state after each round is completed, allowing the state to be recovered even after a crash.
However, this approach introduces a fatal vulnerability: it does not provide interactive DP because the server can clone (i.e., fork) the sealed enclave state at will~\cite{wilde2024forking}. 
If an adversary can fork the process, they can observe outcomes related to step $i$'s noise more than once using the Sybil clients.
This leads to a critical privacy breach even if the $i$th input does not include the target client's values.
This is because the noise added at step $i$ is correlated with the noise from earlier steps.
This correlation can leak information about the noise at earlier steps and, consequently, the earlier private inputs that include the target value.
}

\revision{
Even if the sequence of noise values is predetermined (e.g., from a fixed seed), a malicious server could perform the following forking attack. Suppose the server forks the computation at step $i$.
    It first executes step $i$ on one fork with a cohort that includes a target client's value alongside values from corrupted clients (e.g., all zeros).
    It then executes step $i$ on a second fork using a different cohort containing only corrupted clients (all zeros).
    By comparing the outputs from these two forks, the server can isolate the contribution of the target client. The difference between the outputs effectively removes the common, correlated noise component, thus breaking the privacy guarantee.
}

\revision{
    While combining a TEE's sealing capabilities with a trusted hardware like TPM can enforce state integrity~\cite{strackx2016ariadne}, this integration creates its own security challenges. 
    Securely linking the two trusted components expands the attack surface and opens the door to man-in-the-middle attacks. 
    A prime example is the integration of Intel SGX with a TPM, as explored by Strackx et al.~\cite{strackx2016ariadne}. 
    However, this method is vulnerable to the Cuckoo attack~\cite{parno2008bootstrapping} as mentioned in Introduction.
}

\subsubsection*{Without Sealing}
\revision{
To circumvent forking attacks, one might consider an approach without sealing, where the server restarts the entire process from its initial state if a crash occurs. 
However, this strategy faces a different, equally critical problem in a malicious server model.
}

\revision{
The core issue is that honest clients cannot distinguish whether the server has genuinely crashed or is merely \textit{pretending} to have crashed. 
A malicious server can declare a crash to honest clients while secretly continuing to run the remaining planned iterations with only the corrupted clients it controls. 
This raises the same issue as the case with sealing.
That is, because DP-FTRL uses correlated noise, running these extra, secret iterations leaks further information about the inputs used in the earlier iterations that involved honest clients.
}

\revision{
Consequently, the privacy analysis must account for the worst-case leakage. 
This forces a conservative accounting of the privacy budget for the maximum number of iterations the server could potentially run with its corrupted clients after the apparent stop. 
This results in an inherent and significant privacy loss. 
Furthermore, DP-FTRL's correlated noise mechanism is optimized based on the total number of iterations planned in advance.
An early stop and restart breaks the optimality of this noise structure, potentially leading to a much worse total privacy loss than originally intended. 
Even using SMR~\cite{angel2023nimble,matetic2017rote,howard2023confidential} without sealing suffers from the same problem if a majority of TEEs experience a fault (i.e., disaster).
}

\subsubsection{\revision{Ideal Functionality}}

As discussed in Section~\ref{sec:simulation_based_proof}, we adopt a simulation-based proof framework to prove interactive DP of our system.
\revision{To do so, we first define an ideal functionality that acts as a trusted party and is designed to avoid the issues of the strawman approaches. We then show that our system emulates this ideal functionality with some simulator.}

\paragraph{Ideal Functionality}
\begin{algorithm}[t]
\caption{Ideal Functionality By a Trusted Party}
\label{alg:ideal}
\begin{algorithmic}[1]
\Require Participation schema $\Pi$, the number of rounds $n_\text{round}$, matrix $\mathbf{C}\in \mathbb{R}^{n_\text{round}\times n_\text{round}}$, clipping parameter $\zeta\in \mathbb{R}$, the number of clients $n$, indices of corrupted clients $C_\text{corrupted}$.
\State Receive $D$ from clients.
\State $l \gets$ a sample from a Bernoulli distribution with success probability $\delta_\text{privacy}$.
\If {$l = 1$}
    \State \textbf{send} $D$ to $\text{Sim}$. \quad  // \textit{Leakage}
\Else
    \State Initialize the participation history $H$.
    \State Initialize model parameters $\theta_0\in\mathbb{R}^d$.
    \State Make the matrix with all zeros $\theta_{1:n_{\rm round}}\in\mathbb{R}^{n_{\rm round}\times d}$.
    \State Sample $Z[i,j]\sim\mathcal{N}(0,\sigma^2)$ for $i\in [n_{\rm {round}}]$, $j\in [d]$.
    \For{$i \in [n_{\rm round}]$}
        \State Receive $C_\text{cohort}\subseteq [n], k\in[n_\text{round}], \tilde{\theta}\in\mathbb{R}^d$ from $\text{Sim}$.
        \State Restart from LINE $10$ if any input is invalid.
        \quad // \textit{Valid if $C_\text{cohort}\subseteq f_\text{qualify}(\Pi,H,k)$ and if $k\in [n_{\rm round}]$ is not in $H$ yet.}
        \State Update $H$ with $C_\text{cohort}$ and $k$.

        \State Compute $g_j = \text{clip}(\nabla_\theta \text{loss}(X_j,\tilde{\theta}), \zeta)$ for $j \in C_\text{cohort}$ where $\text{clip}(d, \zeta) = \min(1, \zeta/|d|)d$
        \State For $j \in C_\text{cohort} \cap C_\text{corrupted}$, set $g_j$ as received from $\text{Sim}$
        \State Set $\tilde{\theta}_k := \sum_{j \in C_\text{cohort}} g_j + \zeta(\mathbf{C}^{-1} \mathbf{Z})[k,:]$.
        \State \textbf{send} $\tilde{\theta}_{1:i}$ to $\text{Sim}$ \quad  // \textit{Leakage}.
    \EndFor
\EndIf
\end{algorithmic}
\end{algorithm}

Algorithm~\ref{alg:ideal} outlines the pseudocode for the trusted party in our ideal model.
Two main differences exist when compared to a standard DP-FTRL system, both tailored to the practical needs of a real model implementation.
First, in Lines $2$-$5$, the ideal functionality allows for a small probability of failure that is independent of $D$, indicating a small but non-zero of catastrophic privacy leakage, contributing to $\delta$ in $(\varepsilon,\delta)$-DP.
This adjustment enables the real model to be designed without requiring exhaustive client interactions for each update (i.e., probabilistic relaxation of linearizability), as discussed in Section~\ref{sec:linearizability}.
Second, in Line $11$, the simulator (i.e., an adversary) can modify indices during secure aggregation.
This is due to the non-blocking attribute of linearizability~\cite{herlihy1990linearizability}, which permits concurrent execution.
Also, the adversary can modify the values received from the corrupted clients (Line $15$).

Despite these variations, the following theorem confirms that the original DP guarantee of DP-FTRL is preserved:
\begin{theorem}
\label{theo:ideal}
Assume that if \textsc{Sim} sets $k$ to $i$, $C_{\rm cohort}$ to be a subset of $f_{\rm qualify}(\Pi, H, k)$, $\tilde{\theta}$ to $\tilde{\theta}_{i-1}$, and $g_j$ as defined in Algorithm~\ref{alg:ideal} (i.e., the original DP-FTRL), the view of \textsc{Sim} satisfies $(\varepsilon, \delta+\delta_\text{privacy})$-DP.
Then, the view of any other $\textsc{Sim}^\prime$ satisfies $(\varepsilon, \delta+\delta_\text{privacy})$-DP.
\end{theorem}
The proof is detailed in Appendix~\ref{theo:ideal-appendix}.
Therefore, our goal is to create a concurrent system capable of emulating the ideal functionality within our threat model, whose $\delta_\text{privacy}$ is a sufficiently small probability.
In parallel, the system should be designed to offer availability with small communication and computation overhead for clients.

\subsection{Overview of Our Approach}\label{sec:overview_of_our_approach}

We model $\textsc{ConComp}(M_1, M_2, \dots, M_n)$ as a distributed concurrent system (see Section~\ref{sec:concurrent_system}) to formulate the problem.
In this model, the concurrent system maintains a state object $q$, which includes iteration number $i$, participation history $H$ and noise matrix $\mathbf{Z}$ explained in Section~\ref{sec:dp-ftrl}. 
This state is shared among all clients, and processes are concurrently invoked by the server.
The process with $C_\text{cohort}\subseteq[n]$ and $\texttt{args}_\text{secagg}$ (i.e., indices of participating clients and settings such as current model parameters $\tilde{\theta}$ respectively) increments $q.i$, updates $q.H$ with $C_\text{cohort}\subseteq[n]$ and $q.i$, and outputs the result of stateful aggregation $\textsc{SecAgg}_{q.\mathbf{Z}}$ only if $C_\text{cohort}\subseteq f_\text{qualify}(\Pi, q.H, q.i)$, denoted by \texttt{<q update($C_\text{cohort}$, $q.i$, $\texttt{args}_\text{secagg}$) A><OK($\textsc{SecAgg}_{q.\mathbf{Z}}(C_\text{cohort}, q.i,\texttt{args}_\text{secagg})$) A>}.
$\textsc{SecAgg}_{q.\mathbf{Z}}$ is a function that computes the summation with noise $\zeta(\mathbf{C}^{-1} (q.\mathbf{Z}))[q.i,:]$ for $C_\text{cohort}$ using the settings specified by $\texttt{args}_\text{secagg}$.

\revision{
    Assume that given $\Pi$, a system executes the process \texttt{update} with integrity and linearizability.
    That is, due to the integrity, each process does not break the participation schema $\Pi$ and get the correct result of $\textsc{SecAgg}_{q.\mathbf{Z}}(C_\text{cohort}, q.i,\texttt{args}_\text{secagg})$.
    Due to the linearizability, the result of the processes is equivalent to a sequential processes that respect real-time precedence ordering (i.e., the order of $i$).
    That is, an adversary cannot get multiple $\tilde{\theta}_i$ with the same $i$ and cannot define $i$th $C_\text{cohort}$ and $\texttt{args}_\text{secagg}$ after knowing $\tilde{\theta}_k$ for $i\leq k$.
    Thus, by achieving integrity and linearizability, the system emulates the ideal functionality in Algorithm~\ref{alg:ideal} with a simulator. 
    Therefore, integrity and linearizability of the update process are sufficient conditions for the system to achieve interactive DP.
}

Here, we present the high-level concept of our approach in the real world to achieve integrity and linearizability.
\begin{algorithm}[t]
\caption{The Process}\label{alg:process}
\begin{algorithmic}[1]
\State The server initializes an enclave or replicates the enclave with $q$, $C_\text{cohort}$, and $i$.
\State The enclave sends evidence to the clients.
\State (Audit request) Clients verify the integrity and linearizability using the evidence and send their signatures to the enclave.
\If{the enclave was initialized} the enclave validates these and, if it is the initial enclave, the enclave makes $q.\mathbf{Z}$, is sealed, and stops this process.
\EndIf
\State The enclave verifies $C_\text{cohort}$ and sends evidence to the $C_\text{cohort}$. The server updates $q.H$ with $C_\text{cohort}$ and $i$.
\State (SecAgg request) Clients in $C_\text{cohort}$ verify the evidence and send encrypted data with MAC to the enclave.
\State The enclave verifies these and outputs the result of $\textsc{SecAgg}_{q.\mathbf{Z}}$.
\end{algorithmic}
\end{algorithm}
The pseudocode for the process is provided in Algorithm~\ref{alg:process}.
The server issues two types of the requests to clients: audit request and secure aggregation request.
Intuitively, before outputting the result of \textsc{SecAgg}$_{q.\mathbf{Z}}$—which involves private data and may thus compromise privacy—the server must provide proof of linearizability and integrity with enclaves in the audit request.
Clients must verify this during the audit request, and only upon successful verification can the server complete the process to obtain the result.
Due to the linearizability and integrity checks enforced by clients, a malicious server can only complete linearizable processes.
Thanks to the confidentiality provided by the enclave operations, a malicious server cannot access any information beyond the output.

For verification purposes, a client requires remote attestation and validation of certain values.
The number of clients needed for auditing is modest (e.g., $129$ for $\beta,\gamma=0.1$) as demonstrated in Section~\ref{sec:exp-param}.
Hence, the additional computational and communication overhead is both low and constant.
Moreover, since a different enclave is used for each process, it allows for concurrent execution, enhancing scalability. 
Also, even if an enclave fails, the entire DP-FTRL execution does not halt because the server can make a new process with a new enclave, thereby providing liveness.

This is achieved by enforcing integrity and linearizability through TEEs combined with client auditing.
The challenge lies in proving integrity and linearizability to the clients, a topic that is thoroughly explained in Section~\ref{sec:design_planner}.

\section{Design of Planner}
\label{sec:design_planner}
In this section, we explain how to ensure integrity and linearizability in the presence of enclaves and client auditing.
We start by detailing a system that achieves integrity and linearizability under the assumption that all clients are honest and available (i.e., $(\gamma=0,\kappa=1,\beta=0)$), to highlight the core of our proposed method. 
Then, in Section~\ref{sec:planner-beta-gamma}, we propose modifications to the system to accommodate scenarios where $(\gamma>0,\kappa\neq1,\beta>0)$.

\subsection{Design Under $(\gamma=0,\kappa=1,\beta=0)$}
\label{sec:design_under_the_honest_assumption}

\begin{figure}[t]
    \centering
    \includegraphics[width=0.4\textwidth]{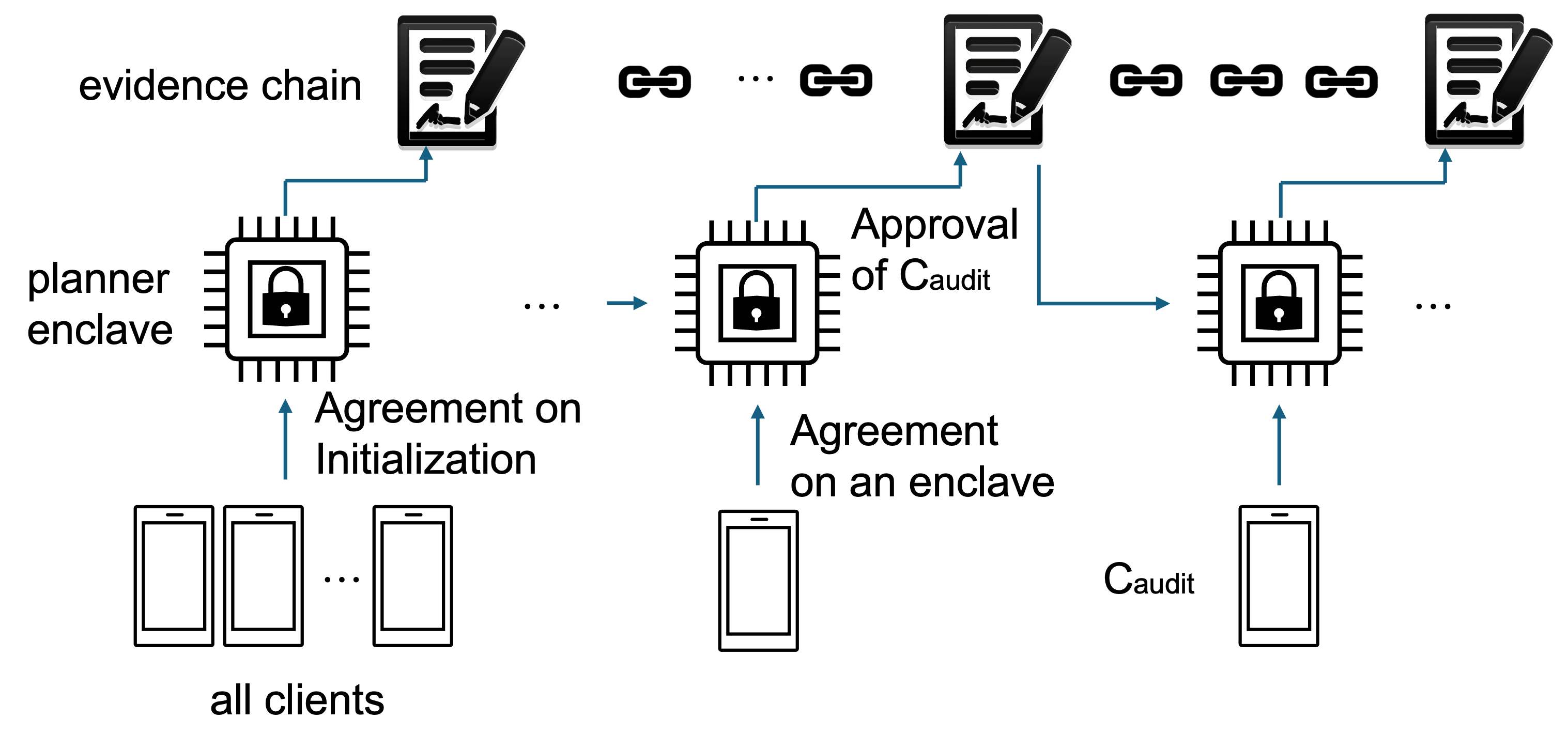}
    \caption{Overview of the core idea behind planner enclaves.}
    \label{fig:intuition}
\end{figure}

Figure~\ref{fig:intuition} illustrates the basic concept of the system. 
We employ an enclave that deploys our code, which we call a planner enclave. 
The planner enclave initially secures consensus from all clients on the shared object $q$ and seals (see Section~\ref{sec:tee}) the states. 
In subsequent processes, replicated planner enclaves load the shared object $q$ from an \textit{evidence chain}, after which a client audits whether the correct shared object was loaded. 
If verified by the client, the planner enclave continues the process and generates evidence designating the next client to audit the subsequent planner enclave. 
This iterative process ensures linearizability, as planner enclaves always load the latest state due to designated client audits (a rollbacked enclave is rejected), and integrity is maintained through the TEE characteristics.

\begin{algorithm}[t]
\caption{Algorithm for the Server-side and the Planner Enclave}
\label{alg:protocol}
\begin{algorithmic}[1]
  \State \textbf{Initialization}
  \State The server initializes a planner enclave with the public keys of all clients $C_{\text{all}} = [n]$, and sets $\text{args}_{\text{selection}} \subseteq C_{\text{all}}$.
  \State The enclave generates random seeds for $\mathbf{Z}$ and a nonce as ID.
  \If{the enclave verifies agreement on $C_{\text{all}}$}
    \State The enclave computes $C_\text{audit} = f_{\text{select}}(\text{args}_{\text{selection}})$.
    \State The enclave seals the states.
    \State The enclave generates evidence including $C_\text{audit}$ and ID.
  \EndIf
  \State The server initializes an evidence chain ($EC$) with the generated evidence.

  \State \textbf{Iteration (concurrent)}
  \State The server selects a cohort $C_{\text{cohort}} \subseteq C_{\text{all}}$.
  \State The server replicates the sealed planner enclave and inputs $EC$, $C_{\text{cohort}}$, $\text{args}_{\text{secagg}}$, and $\text{args}_{\text{selection}}$.
  \If{the enclave verifies $EC$, $C_\text{cohort}$, and $\text{args}_{\text{selection}}$}
    \State The enclave retrieves $C_{\text{audit}}$ from the latest evidence.
    \State The enclave computes $C_\text{next} = f_{\text{select}}(\text{args}_{\text{selection}})$.
    \If{the enclave verifies agreement on $\tau$ clients in $C_{\text{audit}}$}
      \State The enclave generates evidence including $C_\text{next}$ as auditors, $C_{\text{cohort}}$, the hash of $\text{args}_{\text{secagg}}$, and the digest of $EC$.
      \State The server appends the new evidence to the $EC$.
      \State The enclave does $\textsc{SecAgg}_\mathbf{Z}$ on $C_\text{cohort}$ with $\text{args}_{\text{secagg}}$.
      \State The server updates the global model.
    \EndIf
  \EndIf
\end{algorithmic}
\end{algorithm}

\subsubsection{Server-side Algorithm}
\label{sec:server-side algorithm}

Algorithm~\ref{alg:protocol} provides pseudocode for the server and planner enclave protocols. 
At each process (i.e., iteration), the server-side algorithm establishes proofs of linearizability and integrity in Lines $11$-$16$ (i.e., Lines $1$-$3$ of Algorithm~\ref{alg:process}) to conduct secure aggregation in Line $19$ (i.e., Lines $4$ and $6$ of Algorithm~\ref{alg:process}).

Here, $f(\text{args}_\text{selection})$ is a function that outputs $\text{args}_\text{selection}\subseteq [n]$ when $|\text{args}_\text{selection}|\geq1$; otherwise, it aborts, and we set $\tau=|\text{args}_\text{selection}|$.
We use this function as a placeholder for now and will change it later to accommodate the case where $\beta>0$ and $\gamma>0$ (i.e., Section~\ref{sec:planner-beta-gamma}).

Planner enclaves require the server to pass three types of verifications to ensure integrity and linearizability, which are as follows:
\begin{enumerate}
    \item \textit{Initialization} (Line $4$) verifies that all clients have reached consensus on a shared object $q$. This ensures that clients can deny planner enclaves that load a different shared object $q'$ for future auditing.
    \item \textit{Approval of auditors $C_{\text{audit}}$} (Line $13$) verifies that clients $C_{\text{audit}}$ have been approved as auditors by a planner enclave. This means auditors can verify whether a planner enclave loads the correct shared object and will correctly perform secure aggregation (i.e., integrity).
    \item \textit{Agreement on a planner enclave} (Line $16$) verifies that the enclave is the planner enclave agreed upon by auditors $C_{\text{audit}}$. This ensures the enclave maintains integrity and is the sole enclave loading the current shared object (i.e., linearizability).
\end{enumerate}
Next, we discuss the evidence chain that manages a shared object $q$, followed by detailed explanations of the three verifications, illustrated in Figure~\ref{fig:flow_of_proof}.

\paragraph{Evidence Chain}
\label{sec:evidence_chain}
We introduce the evidence chain as the state management method for our concurrent system. 
It consists of a sequence of evidence pieces generated by planner enclaves, with evidences linked via hashes like blockchain~\cite{nakamoto2008bitcoin} (i.e., evidence in the evidence chain has the hash of the previous evidence).
The digest of the evidence chain is the hash of the latest evidence. 
The ID of the evidence chain is the ID included in the initial evidence.
Each piece of evidence includes updates to $q.H$ (in our context, $C_\text{cohort}, i$).

Verification of the evidence chain involves confirming that the latest evidence remains untampered (verified via remote attestation) and that the hash chain is valid.

When a planner enclave successfully verifies an evidence chain and is agreed upon by the auditors $C_{\text{audit}}$ (explained later) documented in the latest evidence, it implies that the enclave reads the latest state by sequentially updating $q.H$ with the evidence chain. 
This occurs because if consensus on $C_{\text{audit}}$ is reached, it is concluded that the current auditors are the most recently approved.

\begin{figure}[t]
    \centering
    \includegraphics[width=0.4\textwidth]{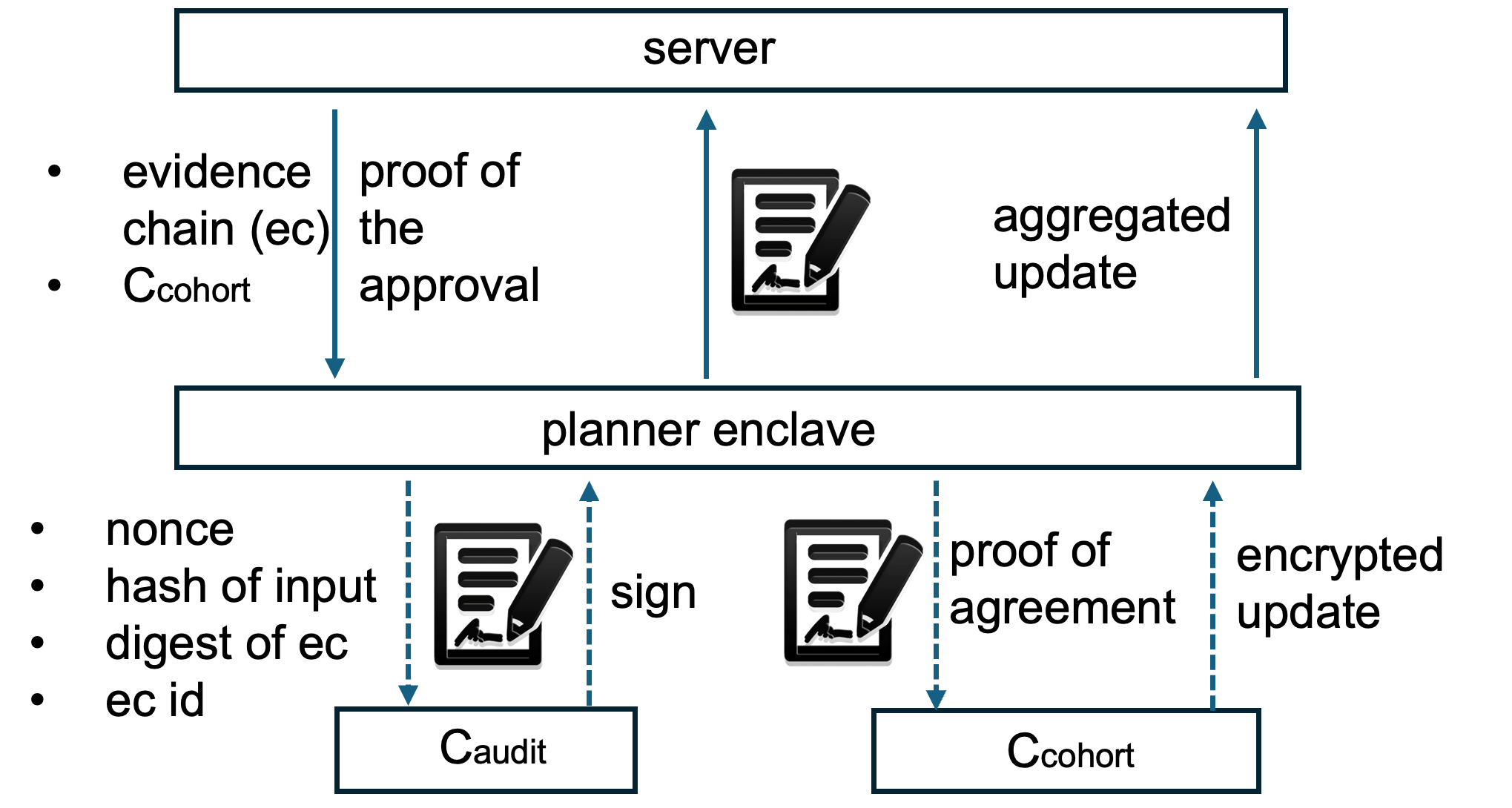}
    \caption{The flow of evidence-related proofs by the server. A dotted line indicates indirect communication via the server.}
    \label{fig:flow_of_proof}
\end{figure}

\paragraph{Approval of Auditors (verified in Line $13$)}

"Approval of auditors" is a process where a planner enclave designates clients $C_\text{audit}$ as auditors just once. 
As shown in Figure~\ref{fig:flow_of_proof}, the server validates this approval through evidence generated by a previous planner enclave (i.e., the latest evidence of the evidence chain). 
A planner enclave is designed to designate clients and produce evidence just once only after it has obtained consensus from the designated auditors. 
Therefore, verification of code via remote attestation establishes that the clients listed as auditors were indeed approved.

\paragraph{Agreement on a Planner Enclave (verified in Line $16$)}
"Agreement on a planner enclave" is the consensus process where auditors designate a specific planner enclave. 
This implies:
\begin{itemize}
    \item Any planner enclave not agreed upon by auditors is invalid.
    \item The planner enclave correctly loads the shared object and will perform secure aggregation without disrupting the participation schema.
\end{itemize}
As illustrated in Figure~\ref{fig:flow_of_proof}, the server first generates evidence using a planner enclave, containing a nonce, the hash of inputs (i.e., the hash of $C_\text{cohort}$ and $\text{args}_\text{secagg}$), the digest of the evidence chain, and the ID of the evidence chain. 
This evidence is provided to the approved auditors retrieved from the evidence chain (i.e., $C_\text{audit}$). 
Approved auditors perform remote attestation to verify the integrity of the evidence and the code.
If auditors verify the ID and that the digest has not been seen before (to prevent replay attacks), they sign the nonce in response. 
The server sends all signatures of $C_\text{audit}$ (i.e., $\tau$ signatures) to the planner enclave to generate further evidence (Lines $17$ and $18$), which proves the "agreement on a planner enclave".

\paragraph{Initialization (verified in Line $4$)}
"Initialization" is a process by which a planner enclave designates the shared object $q$ to all clients. 
Specifically, it implies:
\begin{itemize}
    \item Consensus on the shared object and initial auditors by all clients.
    \item Correct loading of the public key list by the planner enclave.
\end{itemize}
The process is similar to the "agreement on a planner enclave," though the evidence now includes the digest of the public key list, nonce, and the ID of the enclave. 
Each client checks the digest against their known hash of the public key list through PKI, as described in Section~\ref{sec:threat_model}, before registering the nonce as the evidence chain ID and responding with a signature.

Note that "initialization" means all clients agree on $q.\mathbf{Z}$ because a replicated planner enclave possessing the ID as the evidence chain ID also possesses the random seeds for $\mathbf{Z}$.

\paragraph{Secure Stateful Aggregation}
\label{sec:secure_stateful_aggregation}
Here, we explain the secure aggregation of Line $19$  (i.e., $\textsc{SecAgg}_{q.\mathbf{Z}}$).
Due to three prior verifications conducted before Line $19$ of Algorithm~\ref{alg:protocol}, the planner enclave loads the shared object correctly and thus possesses the correct $\mathbf{Z}$, $i$, and $C_\text{cohort}$, where $i$ is the number of entries in the evidence chain.
This allows precise computation of Line 16 of Algorithm~\ref{alg:ideal}.

We primarily adhere to the secure aggregation method outlined by Huba et al.~\cite{huba2022papaya}. 
However, as discussed in Section~\ref{sec:secure_aggregation}, given that the original formulation didn't include noise, additional modifications are required to incorporate stateful noise. 
A simplistic extension suffices; while carrying out secure aggregation method from Huba et al., the enclave adds $\zeta(\mathbf{C}^{-1}\mathbf{Z})[i,:]$. 
Then, aggregated data with the noise is decrypted only when data from $C_\text{cohort}$ is collected with verified MAC.

\subsubsection{Client-side Algorithm}
The client-side algorithm is depicted in Algorithm~\ref{alg:client-side algorithm}. 
The client-side algorithm is stateful, storing {evidence\_chain\_id} and {signed\_digests}. 
It employs two functions: \textsc{Audit} and \textsc{SecureAggregation}.

\textsc{Audit} is involved in Lines $4$ and $16$ of Algorithm~\ref{alg:protocol} (i.e., \textit{Initialization} and \textit{Agreement on a planner enclave}). 
Acting as an auditor, a client verifies whether the planner enclave establishes or reads the shared object. 
During Line $3$ of Algorithm~\ref{alg:client-side algorithm}, a client performs remote attestations to verify that the evidence originated from a planner enclave. 
The evidence may encompass the ID of the evidence chain, a nonce, and the digest of the evidence chain. 
The client verifies whether this matches the stored \texttt{evidence\_chain\_id} (Line $10$). 
If not previously registered (i.e., \texttt{evidence\_chain\_id}=None), it is stored accordingly (Line $8$). 
Then, after ensuring that \texttt{signed\_digests} does not contain the digest to prevent replay attacks (Lines $13$), the client responds with the signature to the nonce.

\textsc{SecureAggregation} is initiated in Line $19$ of Algorithm~\ref{alg:protocol}. 
First, the client performs remote attestation to verify the integrity of the evidence (i.e., to ensure that it is generated by the enclave that has passed the three verifications). 
The client then extracts the peer's public key for the Diffie-Hellman (DH) key exchange from the evidence. 
Using this key, the client sends the update encrypted by a symmetric key, the symmetric key encrypted with the DH key, and the MAC generated from the DH key (see Section~\ref{sec:secure_aggregation} for more details).

\begin{algorithm}[t]
\caption{Client-side Algorithm}
\label{alg:client-side algorithm}
\begin{algorithmic}[1]
\State Initialize evidence\_chain\_id $\gets$ None and signed\_digests $\gets \emptyset$

\Procedure{Audit}{evidence}
    
    \State \textbf{If} \Call{RemoteAttestation}{evidence} fails \textbf{Then} abort
    \State Extract {nonce}, {digest\_of\_evidence\_chain}, {evidence\_chain\_id} from {evidence}
    \If{{self.evidence\_chain\_id} $=$ None}
        \State {self.evidence\_chain\_id} $\gets$ {evidence\_chain\_id}
    \EndIf
    \State \textbf{If} {evidence\_chain\_id} $\neq$ {self.evidence\_chain\_id} \textbf{Then} abort
    \State \textbf{If} {digest\_of\_evidence\_chain} $\in$ {self.signed\_digests} \textbf{Then} abort
    \State \Call{Append}{{self.signed\_digests}, {digest\_of\_evidence\_chain}}
    \State \Return \Call{Sign}{{nonce}}
\EndProcedure

\Procedure{SecureAggregation}{evidence}
    \State \textbf{If} \Call{RemoteAttestation}{evidence} fails \textbf{Then} abort
    \State Extract {nonce}, {$\text{args}_\text{secagg}$}, {ec\_peer\_pub\_key} from {evidence}
    \State {update} $\gets$ \Call{ComputeLocalUpdate}{$\text{args}_\text{secagg}$}
    \State {encrypted\_update}, {encrypted\_decryption\_key}, {ec\_pub\_key}, {MAC} $\gets$
    \State \quad \quad \Call{MakeSecureAggregationMessage}{{update}, {nonce}, {ec\_peer\_pub\_key}}
    \State \Return {encrypted\_update}, {encrypted\_decryption\_key}, {ec\_pub\_key}, {MAC}
\EndProcedure
\end{algorithmic}
\end{algorithm}

\subsubsection{Property of The System}
In our concurrent system, each process involves an object $q$ virtually shared with all clients.
The singular operation applicable to this object is "update" with secure aggregation output, denoted by 
$$
\langle q \; \text{update}(C_\text{cohort}, i) \; A\rangle \langle \text{OK}(\textsc{SecAgg}_{\mathbf{q.Z}}(C_\text{cohort}, i, \text{aux}_\text{secagg}*)), \; A\rangle.
$$ 
While producing $\textsc{SecAgg}_{\mathbf{q.Z}}(C_\text{cohort}, i, \text{aux}_\text{secagg}*)$ as output, the process updates $H$ with $C_\text{cohort}$ and $i$.
This means that Algorithm~\ref{alg:protocol} concurrently performs the process described in Algorithm~\ref{alg:process}.

\begin{theorem}
\label{theorem:linearizability}
Algorithms \ref{alg:protocol} (server-side) and \ref{alg:client-side algorithm} (client-side) achieve both integrity and linearizability with respect to the process described in Algorithm \ref{alg:process} when $\gamma=0$. 
\end{theorem}

The proof is presented in Appendix~\ref{theorem:linearizability-appendix}.
\revision{Also, we formally verified this by the Tamarin prover. See Appendix~\ref{sec:formal_verification} for details}.

\subsubsection{Vulnerability of the System Under $(\gamma>0, \beta>0)$}
\label{sec:vulunarabilities}
Theorem~\ref{theorem:linearizability} demonstrates integrity and linearizability, but its assumptions overlook any client dropouts or dishonest behavior (i.e., $(\beta=0,\gamma=0)$). 
Here, we evaluate the more realistic scenario where $\beta,\gamma>0$, and $\kappa\neq 1$, outlining the vulnerabilities introduced under these conditions.

\paragraph{Sybil Selection}
The inability to differentiate between malicious and honest clients enables Sybil attacks, where Sybils break linearizability. 
A malicious server might select $\text{args}_\text{selection}$ composed entirely of corrupted clients, allowing corrupted auditors to deviate from the protocol, which leads to linearizability violations. 
Specifically, in Line $16$ of Algorithm~\ref{alg:protocol}, corrupted auditors could agree on multiple enclaves, causing arbitrary state forks.

\paragraph{Interruptions Caused by Dropouts}
In Line $16$, the planner enclave requires the collection of signatures from all auditors $C_\text{audit}$ to validate and proceed with secure aggregation, as $\tau=|C_\text{audit}|$. 
However, an auditor dropping out before signing disrupts the entire process outlined in Algorithm~\ref{alg:protocol}. 
The secure aggregation phase in Line $19$ is similarly susceptible to disruptions.

\subsection{Adaptation for $(\gamma>0, \kappa\neq1,\beta>0)$}
\label{sec:planner-beta-gamma}
To address the aforementioned vulnerabilities, we propose necessary modifications to Algorithm~\ref{alg:protocol}.

\subsubsection{Preventing Interruptions by Dropout}
\label{sec:interruption}
As outlined in Section~\ref{sec:vulunarabilities}, client dropouts can halt a process at Line $16$ of Algorithm~\ref{alg:protocol}. 
To alleviate this issue, we modify the value of $\tau$. 
As discussed in Section~\ref{sec:server-side algorithm}, the process of "agreement on a planner enclave" on $C_\text{audit}$ must ensure that no alternative planner enclave gains agreement on $C_\text{audit}$. 
To meet this requirement, we can set $\tau>|C_\text{audit}|/2$, requiring agreement from the majority of $C_\text{audit}$. 
By setting a larger $n_{\text{audit}}=|C_\text{audit}|$ and a smaller $\tau$, the probability of interruptions occurring can be minimized, as described in Section~\ref{sec:availability}.

Interruptions may occur during the secure aggregation process (i.e., Line $19$ of Algorithm~\ref{alg:protocol}) as well.
However, these interruptions are not catastrophic since they occur post-Line $17$, where evidence for the next process is generated. 
This is a feature of linearizability known as non-blocking characteristics~\cite{herlihy1990linearizability}. 
The process, however, loses utility as it fails to output results. 
To tackle this issue, we can utilize an overselection strategy. 
Overselection involves selecting a large pool of candidate clients for secure aggregation, allowing available clients from these candidates to participate.

\paragraph{Remark}
The overselection strategy could result in a disparity between $C_{\text{cohort}}$ and the clients that actually participate, possibly causing data loss since participation schemes are based on $C_{\text{cohort}}$. 
For simplicity, this paper does not consider this utility loss. 
However, this problem is not insurmountable; it can be resolved by creating evidence that includes $C_{\text{part}}$ at Line $19$, where $C_{\text{part}}$ is the set of clients that actually partake in secure aggregation. 
By adding this evidence to the evidence chain, linked to the evidence generated in Line $18$, participation scheme verification can be employed to avert data loss while accommodating client dropouts.

\subsubsection{Preventing Sybil Selection}
Within Algorithm~\ref{alg:protocol}, corrupted clients can only influence the agreement parts at Lines $4$ and $16$ of Algorithm~\ref{alg:protocol}. 
These clients could potentially send signatures without conducting verification during \textsc{Audit}, permitting multiple enclaves to exist carrying the same shared object, disrupting linearizability. 
To counter this, it is crucial to establish that the majority of \textbf{honest} clients in the auditors have reached an agreement.
However, since corrupted clients cannot be differentiated from honest ones, proving this becomes impossible when the number of corrupted clients exceeds the cohort size.

To avoid this issue, we restrict the server's ability to arbitrarily choose auditors. 
The server is allowed to choose a sufficient number of candidates for auditors, but the actual selection of auditors is done randomly by the enclave.
Specifically, the enclave instead employs $f_\text{selection}(\text{args}_\text{selection})$, which is a randomized function that returns a random subset $C_\text{audit}$ of $\text{args}_\text{selection}$ such that $|C_\text{audit}|=n_\text{audit}$ where $n_\text{audit}\in [n]$ is an adjustable parameter and aborts if $|\text{args}_\text{selection}|<n\kappa$.

Despite the random selection, auditors may still include some corrupted clients, making it impossible to unequivocally prove that a majority of \textbf{honest} clients are in agreement. 
Therefore, we set $\tau$ to be a value larger than $n_\text{audit}/2$, ensuring that the probability of honest clients being in the majority becomes sufficiently high.
Note, however, that setting a higher $\tau$ increases the likelihood of interruptions, presenting a trade-off, which is optimized in Section~\ref{sec:parameter-optimization}.

Assuming the server sets $\text{args}_\text{selection}$ as the available clients $C_\text{available}$ (i.e., $|C_\text{available}|=n\kappa$ as per Section~\ref{sec:threat_model}), and considering that up to a ratio of $\gamma$ of all clients may be corrupt, the probability of having a certain number of corrupted clients in $C_\text{audit}$ can be calculated. 
By combining this with $\tau$ and $n_\text{audit}$, the probability of a Sybil attack occurring can be estimated, and this is the probability that breaks integrity and linearizability.

\begin{theorem}
\label{theorem:linearizability_beta}
    When $\gamma>0$, Algorithm~\ref{alg:protocol} maintains linearizability and integrity with a probability of $1-\delta_\text{privacy}=$  
    $$1 - \left(1-\sum_{i=1}^{2(n_{\text{audit}} - \tau)} \binom{n \gamma}{2\tau - n_{\text{audit}} + i} \binom{n(\kappa -  \gamma)}{2n_{\text{audit}} - 2\tau  - i} \Big/ \binom{\kappa n}{n_{\text{audit}}} \right)^{n_{\text{round}}}.$$
\end{theorem}
\revision{
In other words, with probability $1-\delta_\text{privacy}$, our protocol maintains Theorem~\ref{theorem:linearizability} that is formally verified.
}
The proof can be found in Appendix~\ref{theorem:linearizability_beta-appendix}.
The relationship between $\delta_\text{privacy}$ and parameters is shown in~\ref{sec:exp-param}.

\section{Analysis of our system}
\label{sec:analysis}

In this section, we analyze our system from the perspectives of privacy, correctness, liveness, and scalability.

\subsection{Privacy And Correctness}
\label{sec:privacy-analysis}

We first present the main theorem of this paper, which clarifies the privacy and utility aspects of our system.
\begin{theorem}
\label{theo:main_theorem}
Let $M_i$ be Algorithm~\ref{alg:client-side algorithm} for $i\in [n]$. 
The interactive protocol $\textsc{ConComp}(M_1, \ldots, M_n)$, when used with any server protocol, emulates the ideal model (i.e., Algorithm~\ref{alg:ideal}) with some simulator.
\end{theorem}
The proof is presented in Appendix~\ref{theo:main_theorem-appendix}. 
From this theorem, the server protocol can derive the output of MF-DP-FTRL from $\textsc{ConComp}(M_1, \ldots, M_n)$, since Algorithm~\ref{alg:ideal} reflects the MF-DP-FTRL approach by setting $k=i$ and $\tilde{\theta}=\tilde{\theta}_i$. 
$M=\textsc{ConComp}(M_1, \ldots, M_n)$ meets the criteria of $(\varepsilon,\delta+\delta_\text{privacy})$-interactive DP. 
This is because any server protocol interacting with $M$ fulfills $(\varepsilon,\delta+\delta_\text{privacy})$-DP, given that the corresponding simulator itself adheres to $(\varepsilon,\delta+\delta_\text{privacy})$-DP, as concluded in Theorem~\ref{theo:ideal}.

\subsection{Liveness}\label{sec:availability-analysis}
In our context, liveness refers to the capability to complete a process (i.e., Algorithm~\ref{alg:process}).
We analyze three potential issues that might stop the process: crashes, interruptions, and attacks.

\subsubsection{Crashes}
\label{sec:availability-crash}
If an enclave crashes for any reason, there is no straightforward way to recover the process. 
This is because allowing such recovery would enable an adversary to fork the enclave's state.
Maintaining linearizability would require verification, which is not possible.
However, it is important to note that while a single enclave crash may disrupt a process (i.e., Algorithm~\ref{alg:process}), it does not necessarily mean that the entire system (i.e., Algorithm~\ref{alg:protocol}) will come to a halt.
As described in Appendix~\ref{sec:recovery}, even if an enclave crashes, the system can initiate another process to recover the state.
This is an advantage of our system compared to systems relying on SMR that cannot securely recover from disasters~\cite{angel2023nimble, howard2023confidential}.

\subsubsection{Interruptions by dropouts}
\label{sec:availability}
In our system, auditors may drop out, which could potentially cause the system to halt as discussed in~\ref{sec:vulunarabilities}. 
Here, we analyze the probability associated with client dropouts.
\begin{proposition}
\label{prop:availability}
The probability $\delta_\text{interrupt}$ of interruption of auditing is
$
1-\left(1-\sum_{i=1}^{\tau} \binom{n\kappa\beta}{n_{\text{audit}}-\tau+i} \binom{n\kappa(1-\beta)}{\tau-i} \Big/ \binom{n\kappa}{n_{\text{audit}}}\right)^{n_\text{round}}.
$
\end{proposition}
The proof is presented in Appendix~\ref{prop:availability-appendix}.
As shown in Section~\ref{sec:exp-param}, the probability can be small if we set large $n_\text{audit}$.

\subsubsection{Attacks}
We consider scenarios where a malicious client attempts to disrupt the process. 
In our system, a malicious client might send invalid values. 
However, this does not halt the system since the enclave can disregard invalid inputs, similar to handling dropouts. 
Given our system's resilience to dropouts, a single malicious client lacks the capability to stop the process.

\subsection{Scalability}
We evaluate the scalability of our concurrent system from two perspectives. 
First, we examine server-side scalability concerning throughput~\cite{nguyen2022federated}. 
Second, we analyze client-side costs, considering that clients are lightweight and can drop out.


We evaluate the additional communication and computation costs incurred by clients due to our planner. 
In both \textsc{Audit} and \textsc{SecureAggregation}, the communication cost is constant, influenced only by security parameters such as the size of the signature. 
This is discussed in detail in Section~\ref{sec:exp-impl}. 
The computation cost is also constant; the encryption method used is part of the standard communication process and does not necessitate additional computational resources. 
Although the frequency of \textsc{Audit} depends on parameters like $\beta, \kappa$, and $\gamma$, it remains low, as demonstrated in Section~\ref{sec:exp-param}. 
On the contrary, if $n$ becomes large, the frequency of \textsc{Audit} decreases.
Consequently, the costs incurred by clients remain constant even as the system scales.

\subsection{Parameter Selection}
\label{sec:parameter-optimization}
In our system, we have two adjustable parameters: $\tau$ and $n_\text{audit}$. 
Here, we discuss how to determine these parameters.
The goal is to minimize $n_\text{audit}$ to reduce communication costs (i.e., frequency of \textsc{Audit}) while ensuring $\delta_\text{interrupt}\leq p_1\in[0,1]$ and $\delta_\text{privacy}\leq p_2\in[0,1]$ for livenss and security requirements. 
This can be formulated as the following optimization problem:
$$
\min_{\tau, n_\text{audit}} n_\text{audit}\ \text{s.t.}\ \delta_\text{privacy}(n_\text{audit},\tau)\leq p_1\ and\ \delta_\text{interrupt}(n_\text{audit},\tau)\leq p_2,
$$
where $\delta_\text{privacy}(n_\text{audit},\tau)$ and $\delta_\text{interrupt}(n_\text{audit},\tau)$ represent the probabilities of $\delta_\text{interrupt}$ and $\delta_\text{privacy}$ when $\tau$ and $n_\text{audit}$ are provided. 
Since $n_\text{audit}\leq n$ and $\tau\leq n_\text{audit}$ are integers, and the probabilities can be calculated from Proposition~\ref{prop:availability} and Theorem~\ref{theorem:linearizability_beta}, this problem is solvable using binary search.

\section{Experiments}\label{sec:experiments}

\begin{figure}[t]
    \centering
    \begin{subfigure}{0.23\textwidth}
        \centering
        \includegraphics[width=\linewidth]{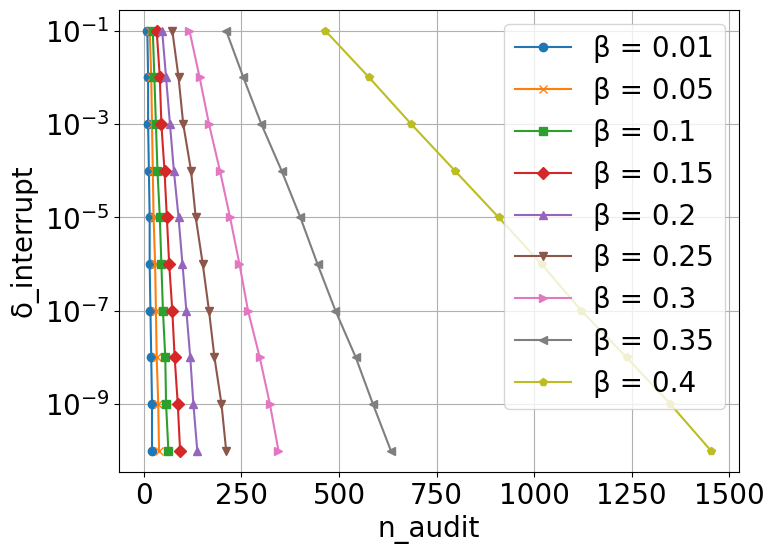}
    \end{subfigure}%
    \begin{subfigure}{0.23\textwidth}
        \centering
        \includegraphics[width=\linewidth]{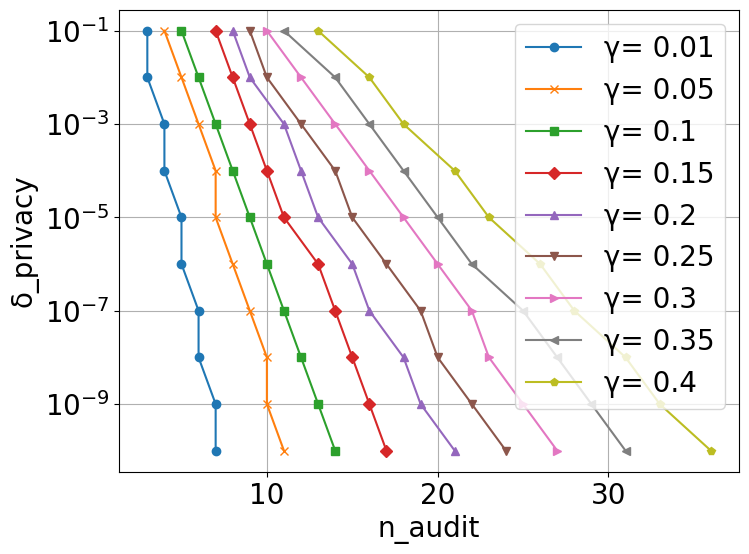}
    \end{subfigure}%
    \caption{The trade-off between communication cost and privacy for various $\gamma$ with $\beta=0$ and $\kappa=1$ (left) and various $\gamma$ with $\beta=0$ and $\kappa=1$ (right).\label{fig:communication-privacy}}
\end{figure}

\begin{figure}[t]
    \begin{subfigure}{0.23\textwidth}
        \centering
        \includegraphics[width=\linewidth]{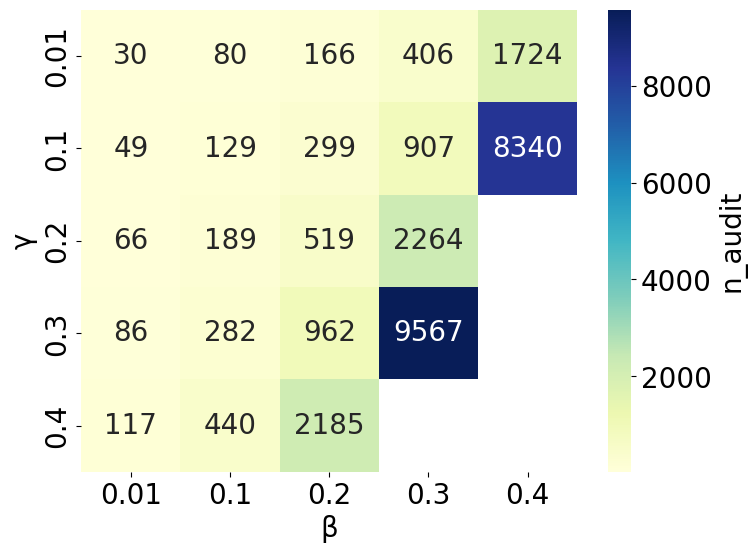}
    \end{subfigure}
    \begin{subfigure}{0.23\textwidth}
        \centering
        \includegraphics[width=\linewidth]{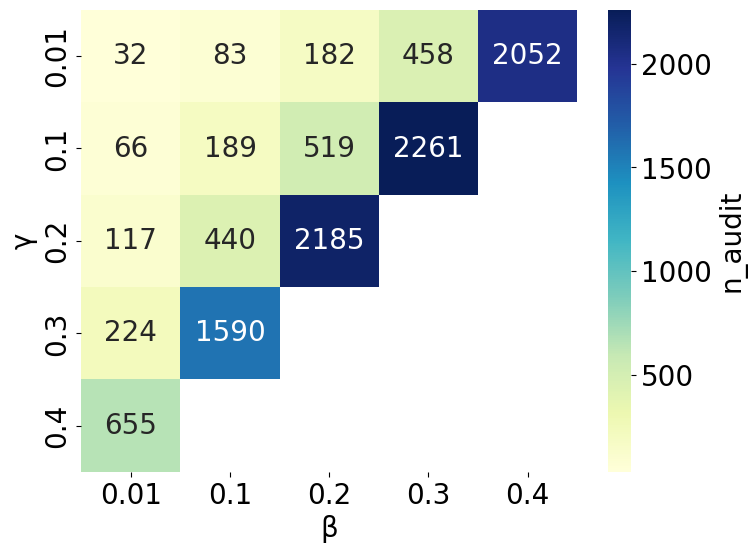}
    \end{subfigure}
    \caption{Communication cost for various settings of $(\beta, \gamma)$ with the constraint of ($\delta_\text{privacy}$, $\delta_\text{interrupt}$)=($10^{-8}$,$10^{-8}$), $\kappa=1$ (left), and $\kappa=0.5$ (right). Results larger than $10000$ are omitted.\label{fig:communication-privacy-availability}}
\end{figure}

In this section, we first examine the effects of various parameters.
Then, we present the TCB size and actual communication overhead associated with the implementation.

\subsection{Communication, Liveness, and Privacy Trade-offs}
\label{sec:exp-param}

In this section, we analyze the trade-offs among communication cost, liveness, and privacy. 
Here, $n_\text{audit}$ represents the minimum number of auditors required for each process derived in Section~\ref{sec:parameter-optimization}, which in turn correlates with the entire communication cost. 
Liveness is quantified by $\delta_\text{interrupt}$, as described in Proposition~\ref{prop:availability}; it represents the probability of system failure in conducting DP-FTRL. 
Privacy is measured by $\delta_\text{privacy}$, as defined in Theorem~\ref{theorem:linearizability_beta}; it denotes the additional overhead to $\delta$ for approximate DP. 
By varying $n_\text{audit}$, we explore the balance among these factors. 
Here, we set $n_\text{round}=10000$ and $n=10^7$, and the server is required to set $|\text{args}_\text{selection}|$ as the available clients.

Initially, we illustrate the balance between communication cost and liveness and between communication cost and privacy in Figure~\ref{fig:communication-privacy}. 
For these analyses, we fix $\gamma=0, \kappa=1$ or $\beta=0,\kappa=1$ to focus on liveness or privacy, respectively. 
The results indicate that by slightly increasing $n_\text{audit}$, the system can significantly lower both $\delta_\text{interrupt}$ and $\delta_\text{privacy}$, implying a communication overhead to achieve sufficient liveness and privacy.

Next, we evaluate the communication cost across various settings of $(\beta, \gamma)$ under the constraints ($\delta_\text{privacy}$, $\delta_\text{interrupt}$)=($10^{-8}$, $10^{-8}$) with $\kappa=1$ or $\kappa=0.5$, as depicted in Figure~\ref{fig:communication-privacy-availability}. 
The plots illustrate the communication cost as $\gamma$ and $\beta$ vary. 
Compared to earlier scenarios where either $\gamma$ or $\beta$ is zero, achieving robustness requires greater communication costs. 
This is because robustness for liveness requires a low value of $\tau$, while robustness for privacy necessitates a high value. 
To increase the value of feasible $\tau$s to simultaneously satisfy both requirements, $n_\text{audit}$ must be increased.

The value of $\kappa$ is also crucial as it quickly increases the value of $n_\text{audit}$. 
This occurs because the candidates for auditors are selected by the server (i.e., adversary), so the ratio of corrupted clients among the candidates rises quickly if the number of candidates (i.e., $n\kappa$) is small. 
Even when $\kappa$ is small, if $\beta$ is also low, our framework can still function effectively with high $\gamma$ and a small $n_\text{audit}$.

We observed that the value of $n$ does not significantly impact the results. 
Thus, by increasing $n$, the communication cost for each client can be reduced because it decreases the probability of being chosen as auditors.

\subsection{Implementations}
We implemented the planner enclave using Intel SGX with the OpenEnclave SDK\footnote{\url{https://github.com/openenclave/openenclave}}. 
For both signing and key exchange, we use ECDSA and ECDH with P-256, respectively, as implemented by OpenSSL.
In this section, we examine the implementation from two angles: the TCB size and the message size required for a client.

\paragraph{TCB size}
We assess the TCB size without the secure aggregation component to compare it with systems focused on achieving linearizability, such as ROTE~\cite{matetic2017rote} and Nimble~\cite{angel2023nimble}, which, to the best of our knowledge, has the smallest TCB size for this purpose. 
Our TCB size is 0.7K excluding libraries, which is more compact than ROTE's 1.1K and Nimble's 2.3K as reported in the paper, as our approach relies on client auditing rather than SMR.

\paragraph{Communication Overhead}
\label{sec:exp-impl}
We evaluated the additional communication overhead involved with \textsc{Audit} and \textsc{SecureAggregation} according to the client-side algorithm. 
For \textsc{Audit}, the evidence size is $5038$ bytes, while for \textsc{SecureAggregation}, it is $5194$ bytes. 
In terms of responses, \textsc{Audit} sends 64 bytes, and \textsc{SecureAggregation} sends 144 bytes. 
Note that we intentionally excluded the size of model updates to concentrate solely on our additional communication costs. 
These communication sizes are fixed and do not vary with FL parameters such as $n$ and $d$.

\section{Conclusion}\label{sec:conclusion}
\revision{
Our paper considers malicious settings in private FL to bridge the gap in privacy principles concerning transparency and verifiability~\cite{daly2024federated}. 
Our approach achieves this using interactive DP, relying exclusively on TEEs and thereby avoiding additional trusted hardware such as a TPM~\cite{strackx2016ariadne}, which is vulnerable to the Cuckoo attack~\cite{parno2008bootstrapping}.
Also, our design follows the fundamental principles of TEE applications~\cite{angel2023nimble,arnautov2016scone}; that is, it has a small TCB size.
Furthermore, our system is resilient to disasters~\cite{angel2023nimble}—meaning it is capable of recovery—and requires only a single TEE-enabled server, thereby reducing management costs.
As for future work, a key direction is to investigate the trade-off between our approach of minimizing trust in the server (i.e., bringing it closer to a zero-trust model) and the implementation costs discussed in Appendix~\ref{sec:practical-implementation-challenges}. 
Exploring simpler client auditing mechanisms with appropriate cost of trust will be crucial for broader practical deployment.
}

\section*{Acknowledgement}
The authors used ChatGPT4o to revise the texts throughout this paper to correct any typos, grammatical errors, and awkward phrasing.

\bibliographystyle{ACM-Reference-Format}
\bibliography{main.bbl}


\appendix

\section{Discussions}\label{sec:discussion}
In this section, we explore the vulnerabilities of our approach, discuss practical implementation challenges.

\subsection{\revision{Side-Channel Attacks}}
\label{sec:side-channel-attacks}

Our protocol's security guarantees depend on the security of the underlying TEEs. 
However, TEEs themselves can be susceptible to various side-channel attacks~\cite{piessens2024side}.
We analyze our protocol's vulnerabilities to such attacks.

Our component for achieving linearizability (i.e., planner) is independent of the secure aggregation mechanism. 
This design offers the flexibility to employ different secure aggregation methods: SMPC~\cite{ball2024secure} or TEE-based approaches~\cite{huba2022papaya}. 
We now analyze the side-channel vulnerabilities of our system in two scenarios: using TEE-based secure aggregation and using SMPC-based secure aggregation.

\paragraph{Side-channel attacks that break confidentiality}
One well-known category of side-channel attacks is Controlled-Channel Attacks (CCA)~\cite{xu2015controlled}, where an adversary can infer sensitive input by observing an application's input-dependent execution flow through architectural side-effects like page faults. 
Our linearizability component can be input-oblivious with respect to the execution flow and is therefore not vulnerable to CCA. 
However, a TEE-based secure aggregation protocol can be vulnerable if it performs input-dependent operations, such as data sparsification~\cite{kato2023olive}. Consequently, the secure aggregation protocol must be carefully designed to avoid such vulnerabilities.

Next, we consider microarchitectural side-channel attacks (e.g., transient execution attacks such as Spectre~\cite{kocher2020spectre} and Meltdown~\cite{lipp2020meltdown}) that leak memory content from within an enclave. 
The recovery protocol for our planner mechanism depends on enclave memory confidentiality. 
However, this dependency can be removed at some cost, as shown in Appendix~\ref{subsec:recovery_transparent}. TEE-based secure aggregation is inherently vulnerable to this attack class because it processes raw data in plaintext inside the enclave. 
An alternative is to use an SMPC-based secure aggregation protocol~\cite{ball2024secure}, which does not require enclave confidentiality. 
This approach, however, introduces a different vulnerability: it is susceptible to Sybil attacks. Since SMPC requires clients to collaboratively add noise, an adversary controlling a set of corrupted clients could instruct them to omit their noise contributions, thereby breaking the DP guarantee of the final output. 
This implies the need for a mechanism that can securely add noise without relying on enclave confidentiality.

\paragraph{Side-channel attacks that break integrity}
Our protocol's linearizability guarantee is critically dependent on the integrity of the enclave. 
If an attacker compromises the enclave's integrity, the entire protocol becomes vulnerable. 
For instance, fault injection attacks like Plundervolt~\cite{murdock2020plundervolt} or Voltpillager~\cite{chen2021voltpillager} could be used to bypass integrity and linearizability checks. 
In such a scenario, the protocol would fail, allowing an adversary to control the process, extract noise information, and ultimately breach interactive DP.
Such side-channel attacks typically necessitate either physical access to the hardware~\cite{chen2021voltpillager} or kernel-level privileges~\cite{murdock2020plundervolt} for execution. Within a robust cloud VM infrastructure, direct physical access is inherently prevented. Moreover, the hypervisor acts as a critical defense layer, blocking privileged operations that attempt to manipulate the physical hardware state—for example, by intercepting writes to MSRs~\cite{chen2021voltpillager}. Consequently, leveraging a cloud-based VM is a valid strategy for mitigating the risks posed by this class of side-channel attacks

Furthermore, microarchitectural attacks like SGAxe~\cite{vansgaxe} could potentially leak the attestation key, compromising integrity.
Such vulnerabilities must be promptly addressed by TEE manufacturers, so we can mitigate the risk by performing frequent TCB updates, and clients can verify that the server is using a secure and up-to-date version of the TEE.

\subsection{\revision{Implementation Challenges}}
\label{sec:practical-implementation-challenges}

Our system's security guarantee does not rely on additional trusted hardware like TPMs or complex mechanisms like SMR, but rather on honest clients performing auditing. This design choice enhances transparency. However, it also introduces new practical implementation challenges, particularly concerning the deployment of auditing functions on client devices. A flawed or unavailable implementation of this client auditing is critical, as it could halt the entire FL protocol. Here, we discuss the specific challenges of our implementation and potential mitigation strategies.

\subsubsection{Remote Attestation on Heterogeneous Clients}

A core requirement of our protocol is that clients perform remote attestation (RA) to verify the server-side enclave's quote. Implementing this verification process on client devices presents a challenge, especially in a typical cross-device FL environment with heterogeneous clients.

Verifying an RA quote can be a complex task. For instance, with Intel SGX, this may involve processing collateral information and requires the integration of specific SDKs and cryptographic libraries. This increases the size and complexity of the client-side application, potentially leading to instability and harming the availability of client auditing.
    
A practical solution is to employ a third-party verifier service like Intel Tiber Trust Authority\footnote{\url{https://www.intel.com/content/www/us/en/security/trust-authority.html}}. Instead of performing the complex verification locally, the client application can forward the enclave's quote to this trusted service, which then returns a simple pass/fail result. This approach offloads the complexity from the diverse client devices, simplifying the client application and improving overall system stability and availability.

\subsubsection{TCB Update and Management}

In a real-world deployment, the TCB will require updates to address bugs or respond to newly discovered side-channel vulnerabilities. This necessity poses a challenge: clients must have a secure mechanism to learn the correct hash of the new, updated enclave binary.

This is a common challenge for TEE-based systems. For example, Papaya~\cite{huba2022papaya} addresses this by using a verifiable log, managed by a trusted party, to record valid enclave hashes. Our system can address this challenge in a more integrated and simpler manner by leveraging our existing \textbf{evidence chain}.

The evidence chain in our protocol is already an externally verifiable mechanism for managing state. We can extend its use to manage TCB updates. 
The FL provider appends a special entry to the evidence chain containing the hash of the new valid enclave version.
Clients, who already verify the evidence chain, can securely learn the new TCB hash from this entry.

This approach seamlessly integrates TCB management into our existing state-management protocol, eliminating the need for a separate verifiable log system. This highlights a key advantage of our system's design: the secure and verifiable state management provided by the evidence chain offers a simple and robust solution for TCB updates.

\section{Recovery}
\label{sec:recovery}

This appendix details two proposed recovery mechanisms designed to handle a critical failure scenario: an enclave crash occurring during the state transition process. 
Our protocol's liveness depends on the successful generation of a new evidence block that designates the next set of auditors. 
A crash during this critical phase (specifically, within Line 16 of Algorithm~\ref{alg:protocol} after auditor signatures have been collected but before the next state is committed) could halt the entire DP-FTRL execution.

\revision{
To solve this, we propose two distinct solutions, each tailored to a different threat model regarding the security guarantees of the underlying TEE.
}

\subsection{Recovery under a Confidentiality Assumption}
\label{subsec:recovery_confidential}

This first solution operates under the confientiality assumption that the TEE provides robust confidentiality for enclave memory, protecting it from both software and hardware-based attacks, including side-channel attacks.

\paragraph{Mechanism}
The protocol is augmented with the following steps:
\begin{enumerate}
    \item \textbf{Key Initialization:} During the initial setup (Line 2 of Algorithm~\ref{alg:protocol}), the planner enclave generates a symmetric \texttt{recovery\_key}. 
    This key is immediately sealed and persisted as part of the enclave's initial state, never being exposed outside the trusted boundary.
    
    \item \textbf{State Commitment:} In a regular update process, before requesting auditor signatures, the planner enclave randomly determines the next set of auditors, \texttt{C\_{next}}.
    It then encrypts this selection with the \texttt{recovery\_key} and includes this encrypted blob within the evidence presented to the current auditors.
    
    \item \textbf{Recovery Trigger and Execution:} If the enclave crashes during Line 16 of Algorithm~\ref{alg:protocol} after collecting a sufficient number of signatures from \texttt{C\_{audit}}, the server initiates a recovery. 
    It launches a new planner enclave instance, invoking a dedicated \texttt{recovery} function. 
    The server provides this function with the evidence from the crashed enclave (containing the encrypted \texttt{C\_{next}}) and the collected auditor signatures.
    
    \item \textbf{State Restoration:} The recovery enclave unseals the \texttt{recovery\_key}, decrypts \texttt{C\_{next}}, and verifies the validity of the provided signatures against the auditor set specified in the evidence. 
    If successful, it deterministically regenerates the evidence block that the previous enclave failed to produce, thus restoring the continuity of the evidence chain and allowing the overall FL process to continue.
\end{enumerate}

\paragraph{Limitations and Trade-offs}
\revision{
While this approach effectively recovers from a crash, it has two important limitations. 
First, the utility of the crashed round is partially lost, as the secure aggregation result itself cannot be recovered; only the protocol's state continuity is restored.
}

\revision{
Second, and more critically, the security of this mechanism is entirely dependent on the confidentiality of the \texttt{recovery\_key}.
As discussed in Appendix~\ref{sec:side-channel-attacks}, if a side-channel attack were to leak this key, an adversary could decrypt $C_\text{next}$ prematurely. 
This would create a strategic vulnerability.
By revealing the upcoming random selection of auditors, it allows the adversary to repeatedly force a re-selection by inducing crashes until a favorable set (i.e., one including many corrupted clients) is chosen.
}

\subsection{\revision{Recovery Resilient to Side-Channel Attacks}}
\label{subsec:recovery_transparent}

To address the limitations of the confidentiality-dependent approach, we propose a second mechanism designed for a weaker threat model that assumes a \textit{transparent enclave}~\cite{pass2017formal}. 
This model presumes that an adversary may be able to read the enclave's memory, rendering secret-based commitments insecure.

\paragraph{Mechanism}
This protocol avoids long-term secrets and instead relies on a modification to the protocol logic:

\begin{enumerate}
    \item \textbf{No Secret Key:} This mechanism entirely avoids the use of a secret \texttt{recovery\_key}. 
    The commitment to the next state is purely logic-based.
    \item \textbf{Recovery Flag:} The evidence is extended to include a boolean \texttt{recovery\_flag}. 
    In a normal, non-failed operation, this flag is always set to \texttt{false}.
    \item \textbf{Controlled Fork on Post-Signature Crash:} 
    The protocol explicitly permits a controlled, temporary fork to handle a specific failure scenario: a crash that occurs after the auditors for a given round (e.g., $(i-1)$-th $C_{\text{audit}}$) have signed for a normal transition (to state $S_i$ with new randomly chosen auditors $i$-th $C_{\text{audit}}$), but before that state is finalized. 
    In this case, $(i-1)$-th $C_{\text{audit}}$ are permitted to sign a second time for a recovery transition (to state $S^\prime_i$), provided the server's request is for a recovery process and the corresponding enclave sets the \texttt{recovery\_flag} to \texttt{true}. 
    In this recovery transition, $i$-th $C_{\text{audit}}$ is set to the same set as $(i-1)$-th $C_{\text{audit}}$, and secure aggregation is not performed. 
    This allows two potentially valid successor states, $S_i$ and $S^\prime_i$, to co-exist, representing a deliberate fork in the state history. 
    \item \textbf{Fork Resolution via Interactive Verification:}
    To resolve this fork and ensure the system converges back to a single history, the subsequent planner enclave is tasked with additional verification before it can proceed. 
    To validate a transition from either $S_i$ or $S^\prime_i$, the planner enclave must query the auditors of the previous round ($(i-1)$-th $C_{\text{audit}}$). 
    Each member of $(i-1)$-th $C_{\text{audit}}$, based on their local knowledge of whether they participated in a recovery signature, provides a signed message to the planner enclave endorsing only one of the two branches. 
    The planner enclave must collect a sufficient quorum of these endorsements to prove that a single, unambiguous branch has been chosen. 
\end{enumerate}

\paragraph{Limitations and Trade-offs}
The key advantage of this design is its resilience to confidentiality-breaking attacks like side-channels. 
However, this robustness comes at a price.
The protocol becomes significantly more complex. 
More critically, it introduces new liveness dependencies that may, paradoxically, reduce overall system availability. 
The success of a recovery now hinges on the availability of \textit{two} consecutive sets of auditors (the current and the previous).
A failure or mass dropout in either of these sets would cause the auditing itself to fail, making the system potentially more brittle than the single-dependency model in Section \ref{subsec:recovery_confidential}.

\section{Related Work}
\label{sec:related work}
The most relevant recent works include those by Bienstock et al.~\cite{bienstock2024dmm} and Ball et al.~\cite{ball2024secure}, which attempt to realize DP-FTRL in a (partially) malicious setting.
These works utilize SMPC and reshare secrets between cohorts to achieve correlated noise.
However, they assume that Sybil attacks do not occur, which means that they partially rely on the honesty of the central server during the planning phase to select the correct $C_\text{cohort}$. 
To the best of our knowledge, we are the first to propose a maliciously secure DP-FTRL that includes planning.
In the following, we discuss related works that share similar objectives or approaches.

\subsection{Secure FL}
There is a vast body of work focusing on secure FL under untrusted server conditions. 
The foundational study was conducted by Bonawitz et al.~\cite{bonawitz2017practical}, which proposed using SMPC within a cohort for aggregating updates in FL such as FedAvg to prevent untrusted servers from accessing individual updates. 
Some research has extended this to PFL setting that adds noise~\cite{kairouz2021distributed, agarwal2021skellam}. 
Subsequent work~\cite{bell2020secure, li2023lerna, ma2023flamingo, so2021turbo} focused on optimizing the number of intermediate helper users or minimizing communication rounds, but this requires additional computation, communication, and synchronization from clients. 
Moreover, corrupted clients could refuse to add distributed noise, undermining the intended privacy guarantees.

SMPC among non-colluding servers~\cite{talwar2024samplable, corrigan2017prio} may overcome these challenges but requires non-collusion assumptions, which are unrealistic for a single FL conductor.

There is also literature on TEE-based secure aggregation for FL~\cite{huba2022papaya, mo2021ppfl, chamani2020mitigating}. 
However, these studies focus only on stateless and noiseless secure aggregation methods like FedAvg, and the challenge of achieving DP-FTRL using TEEs has not yet been addressed.

\subsection{Maliciously Secure Federated Analytics}
As a broader concept of secure FL, (maliciously) secure federated analytics for DP has become a recent trend~\cite{google-fa}. 
One of the major challenges is conducting stateful workloads on TEEs.

Confidential Federated Computation (CFC)~\cite{eichner2024confidential}, although not specifically designed for FL, shares similar motivations with our work. 
It addresses rollback attacks using SMR through multiple TEE.
Therefore, it is susceptible to the problems mentioned in Introduction~\ref{sec:intro}.
Also, the application of this general system to DP-FTRL remains unclear. 
Srinivas et al.~\cite{srinivas2024federated} proposed a federated analytics system that utilizes TEEs optimized for one-shot analytics to avoid the challenge of the state management, but this approach sacrifices state management, making DP-FTRL unattainable.

The SMPC-based approach for sparse histogram~\cite{braun2024malicious}, while not explicitly for federated learning, shares a similar environment.
They secure the sparse histogram mechanism with interactive DP, while our approach secures DP-FTRL.

\subsection{Maliciously Secure State Management System}
A crucial aspect of our system is realizing a linearizable concurrent system under malicious conditions, a concept explored in previous research. 
Generally, two methods achieve this: using TEEs and utilizing auditors.

\subsubsection{TEE Based Approach}
\label{sec:tee-based-state-management}
Many TEEs provide sealing capabilities to offload encrypted and signed state to disk, which, combined with monotonic hardware counters, could be used to implement rollback protection~\cite{strackx2016ariadne}. 
However, this approach suffers from low performance~\cite{matetic2017rote}, weak security~\cite{parno2008bootstrapping}, and quick exhaustion of hardware counters, hence recent Intel SGX versions lack monotonic counters. 
Systems using SMR for linearizability~\cite{matetic2017rote, angel2023nimble} or Raft consensus for fault tolerance~\cite{howard2023confidential, wang2022engraft} on TEEs might mitigate the problem.
Elephants DP~\cite{jin2024elephants} utilizes similar systems for DP budget management, but it assumes a trusted data curator or does not address FL.
Replacing our planner's linearizability component with these systems could work but would increase the TCB size significantly, complicating verification. 
Furthermore, these solutions cannot recover from disasters~\cite{angel2023nimble, howard2023confidential}, making them costly to manage. Our planner can operate on a single machine and recover.

\subsubsection{Auditor Based Approach}
In auditor-based methods, systems provide external state verifiability~\cite{yue2023glassdb, yang2020ledgerdb}. 
The main challenge is establishing trusted auditors and process integrity. 
Combining TEEs might offer some integrity, but corrupted auditors could lead to faults and privacy leaks. 
Integrating blockchain~\cite{niu2022narrator, das2019fastkitten} with TEEs is another way to maintain consistent state. 
However, managing a blockchain is required, and in a single FL conductor setting, adopting a permissioned blockchain (i.e., internal verifiability) is infeasible, while managing a permissionless blockchain poses incentive issues or is vulnerable to a $51\%$ attack.

\section{Missing Proofs}
\label{sec:proofs}

\begin{theorem}[Theorem~\ref{theo:ideal}]
\label{theo:ideal-appendix}
Assume that if \textsc{Sim} sets $k$ to $i$, $C_{\rm cohort}$ to be a subset of $f_{\rm qualify}(\Pi, H, k)$, $\tilde{\theta}$ to $\tilde{\theta}_{i-1}$, and $g_j$ as defined in Algorithm~\ref{alg:ideal} (i.e., the original DP-FTRL), the view of \textsc{Sim} satisfies $(\varepsilon, \delta+\delta_\text{privacy})$-DP.
Then, the view of any other $\textsc{Sim}^\prime$ satisfies $(\varepsilon, \delta+\delta_\text{privacy})$-DP.
\end{theorem}

\begin{proof}
This proof uses the "add the delta" technique~\cite{add-the-delta} (i.e., joint convexity).
We interpret the mechanism $M(D)$ as choosing either $M^\prime(D)$ with probability $1-\delta_\text{privacy}$ or $M^{\prime\prime}(D)$ with probability $\delta_\text{privacy}$, where $M^{\prime\prime}$ is defined by Line $4$ of Algorithm~\ref{alg:ideal} and $M^{\prime}$ by Lines $6$-$17$. If $M^\prime$ satisfies $(\varepsilon, \delta)$-DP, then $M$ satisfies $(\varepsilon, \delta+\delta_\text{privacy})$.

We will demonstrate that $M^\prime$ satisfies $(\varepsilon, \delta)$-DP. 
When setting $k$ to $i$, $C_{\rm cohort}$ to be a subset of $f_{\rm qualify}(\Pi, H, k)$, $\tilde{\theta}$ to $\tilde{\theta}_{i-1}$, and $g_j$ as defined in Algorithm~\ref{alg:ideal}, this aligns with the matrix mechanism in adaptive streaming~\cite{denisov2022improved}. 
We show that the privacy guarantee holds even if $k$, $C_{\rm cohort}$, $\tilde{\theta}$, and $g_j$ are chosen adaptively, using the method described in Theorem 2.1 of Denisov et al.~\cite{denisov2022improved}, which equates the matrix mechanism to the adaptive Gaussian mechanism. 
Given that the sensitivity of $M^\prime$ is bounded by the participation schema due to Line $10$ and \textsc{Sim} receives $\tilde{\theta}_{1:i}$ at round $i$ (i.e., $\mathbf{C}\theta_{1:i}+\mathbf{Z}[:i]$ where $\theta_{1:i}\in\mathbb{R}^{i \times d}$) instead of $\tilde{\theta}_k$, our mechanism remains equivalent to the adaptive Gaussian mechanism. 
Consequently, the folklore technique used in~\cite{denisov2022improved} implies that the privacy of the adaptive Gaussian mechanism is equivalent to that of the non-adaptive Gaussian mechanism.
Therefore, the privacy guarantee is maintained just as in that of Denisov et al.~\cite{denisov2022improved}.
Thus, we conclude that $M^\prime$ satisfies $(\varepsilon, \delta)$-DP, and consequently, $M$ satisfies $(\varepsilon, \delta+\delta_\text{privacy})$-DP.

\end{proof}

\begin{theorem}[Theorem~\ref{theorem:linearizability}]
\label{theorem:linearizability-appendix}
Algorithms \ref{alg:protocol} (server-side) and \ref{alg:client-side algorithm} (client-side) achieve both integrity and linearizability with respect to the process described in Algorithm \ref{alg:process} when $\gamma=0$.
\end{theorem}

\begin{proof}
This derives from the three verifications outlined in Section~\ref{sec:design_under_the_honest_assumption}. 
Here's a brief review:

\textbf{Integrity:} 
We demonstrate that each process (i.e., Algorithm~\ref{alg:process}) loads the shared object $q$ and correctly executes $\textsc{SecAgg}_{q.\mathbf{Z}}$ on $C_\text{cohort}$ with index $i$ only when $C_\text{cohort}$ adheres to the participation schema with $q.H$ and $i$, and correctly updates $q.H$ in the process.

Firstly, we show that a completed process loads the shared object $q$ due to the evidence chain for $q.H$ and the sealing for $q.\mathbf{Z}$.
Initially, all clients store the evidence chain ID corresponding to $q.\mathbf{Z}$ (Line 4 of Algorithm~\ref{alg:protocol} in the verification of "Initialization"). 
An enclave loads $q.\mathbf{Z}$ by unsealing with the evidence chain ID and generates evidence, including it.
During \textsc{Audit}, clients verify that the evidence chain ID inside the evidence matches theirs. 
Thus, the enclave can pass the "agreement on an enclave" stage only when $q.\mathbf{Z}$ is correctly loaded with the evidence chain ID by unsealing. 
The enclave retrieves $q.H$ from the evidence chain, and the validity of this chain is verified by remote attestation to extract correct auditors ("the approval of auditors") and \textsc{Audit} to verify the evidence chain's digest.
Thus, the enclave can pass the "agreement on an enclave" stage only when $q.H$ is correctly loaded from the evidence chain.

Next, we ensure the correct execution of $\textsc{SecAgg}_{q.\mathbf{Z}}$.
In \textsc{Audit}, auditors agree only when they verify, through remote attestation, that $C_\text{cohort}$ is confirmed to adhere to the participation schema and that $\textsc{SecAgg}_{q.\mathbf{Z}}$ will be executed correctly.
Therefore, passing the "agreement on an enclave" stage indicates the correct execution of $\textsc{SecAgg}_{q.\mathbf{Z}}$.

\textbf{Linearizability:}  
In this section, to show linearizability, we demonstrate that the result of concurrent processes is equivalent to the result produced by processes that are sequentially (and not concurrently) invoked and completed based on the length of their loaded evidence chain. 
This equivalence is due to the auditors' approval and the the agreement on an enclave.
To update $q.H$ (i.e., to create new evidence for the evidence chain as the enclave loads $q.H$ from the evidence chain), the evidence chain requires the signatures of the auditors recorded in the latest evidence. 
The auditors sign only once, indicating that only one enclave, which has loaded the evidence chain, can pass the verification. 
That is, other enclaves that load the same evidence chain (i.e., $q$) will never complete their process. 
From the integrity, an enclave that loads $q$ completes the process using the loaded $q$. 
Thus, the result of concurrent processes aligns with that of processes sequenced according to the order defined by the length of the loaded evidence chain.

\end{proof}

\begin{theorem}[Theorem~\ref{theorem:linearizability_beta}]
\label{theorem:linearizability_beta-appendix}
    When $\gamma>0$, Algorithm~\ref{alg:protocol} maintains linearizability and integrity of Theorem~\ref{theorem:linearizability} with a probability of $1-\delta_\text{privacy}=$  
    $$1 - \left(1-\sum_{i=1}^{2(n_{\text{audit}} - \tau)} \binom{n \gamma}{2\tau - n_{\text{audit}} + i} \binom{n(\kappa -  \gamma)}{2n_{\text{audit}} - 2\tau  - i} \Big/ \binom{n\kappa}{n_{\text{audit}}} \right)^{n_{\text{round}}}.$$
\end{theorem}

\begin{proof}
In each completed process, the system satisfies Theorem~\ref{theorem:linearizability} as long as there is a majority agreement among honest clients within $C_{\text{audit}}$ (i.e., reaching consensus). 
Therefore, the probability we are interested in calculating is the probability that at least one round fails to reach consensus among the $n_{\text{round}}$ processes completed.

First, let's calculate the probability that a single round fails to reach a majority agreement among honest clients in $C_\text{audit}$. 
This situation occurs when the collected $\tau$ signatures do not include a majority of honest clients' signatures within $C_{\text{audit}}$. 
Define $a$ as the random variable representing the number of corrupted clients among those in $C_{\text{audit}}$.
When $\tau$ signatures are collected, they include at most $\tau-a$ signatures of honest auditors. 
If this number is less than the number of honest clients in $C_\text{audit}$ who have not agreed, i.e., $n_\text{audit}-\tau$ where $n_\text{audit}=|C_\text{audit}|$, then a majority agreement is not reached. 
Hence, the probability is given by:

$$
\Pr[\tau-a < n_{\text{audit}}-\tau] = \Pr[a > 2\tau - n_{\text{audit}}].
$$

Here, $a$ follows a hypergeometric distribution because it counts the number of corrupted clients among a randomly selected $n_{\text{audit}}$ from the $n_{\text{available}} = n\kappa$ candidates, which includes at most $n \times \beta$ corrupted clients. Therefore, the probability $p$ we want to find is:
$$
p = \Pr[a > 2\tau - n_{\text{audit}}] = \sum_{i=1}^{2(n_{\text{audit}} - \tau)} \Pr[a = 2\tau - n_{\text{audit}} + i].
$$

$$
= \sum_{i=1}^{2(n_{\text{audit}} - \tau)} \binom{n \gamma}{2\tau - n_{\text{audit}} + i} \binom{n(\kappa -  \gamma)}{2n_{\text{audit}} - 2\tau  - i} \Big/ \binom{n\kappa}{n_{\text{audit}}}.
$$
This occurs independently across $n_{\text{round}}$ rounds. Let $b$ be the random variable that represents the number of rounds where consensus was not reached. 
Then, the probability of interest is:

$$
1 - \Pr[b = 0].
$$

Since $b$ follows a binomial distribution with probability $p$ and $n_{\text{round}}$ trials, the probability we are looking for is:

$$
1 - (1-p)^{n_{\text{round}}}.
$$
This expression gives us the probability that at least one of the rounds fails to reach a consensus.

\end{proof}

\begin{theorem}[Theorem~\ref{theo:main_theorem}]
\label{theo:main_theorem-appendix}
$M=\textsc{ConComp}(M_1, \ldots, M_n)$, when used with any server interactive protocol, emulates the ideal model (i.e., Algorithm~\ref{alg:ideal}) with some simulator.
\end{theorem}

\begin{proof}
Let $T(D)$ denote the ideal functionality executed by a trusted party as defined in Algorithm 2.
Here, we show that for any interactive protocol $A$ in the real world, there exists a simulator $\text{Sim}$ in the ideal world such that the view of the adversary, $\texttt{View}(A,M(D))$, is computationally indistinguishable from the view of the simulator, $\texttt{View}(\text{Sim}, T(D))$.

First, let us analyze the adversary's view in the real world, $\texttt{View}(A, M(D))$.
This view consists of all messages that the adversary $A$ receives during the execution of the protocol $M(D)$. 
Clients perform remote attestation to verify the integrity of the enclave code and its parameters. They reject any messages not originating from a legitimate enclave, ensuring that any data they accept is a valid output of a process defined in Algorithm~\ref{alg:process}.
Although the adversary $A$ controls when enclaves are executed, it cannot access their internal memory due to the confidentiality guarantees of the TEE. Consequently, $A$ can only observe the inputs and outputs of the enclave for each process.
Crucially, according to Theorem~\ref{theorem:linearizability_beta}, the system maintains linearizability with a probability of $1 - \delta_{privacy}$. This means that with overwhelming probability, the sequence of outputs generated by the enclaves is equivalent to a sequential history where each process 
$$
\langle q \; \text{update}(C_\text{cohort}, i) \; P\rangle \langle \text{OK}(\textsc{SecAgg}_{\mathbf{q.Z}}(C_\text{cohort}, i, \text{aux}_\text{secagg}*)) \; P\rangle,
$$
completes in an ordered manner, with $i$ incrementing sequentially. With the remaining probability of $\delta_{privacy}$, linearizability may be violated, potentially exposing information about the dataset $D$ to the adversary.


Now, let's construct the simulator $\text{Sim}$ for the ideal world. 
$\text{Sim}$ must generate a view, $\texttt{View}(\text{Sim}, T(D))$, that is computationally indistinguishable from $\texttt{View}(A, M(D))$. 
$\text{Sim}$ can simulate the messages sent from clients to the server (e.g., signatures, encrypted updates) by generating random values. 
This is justified because, under standard cryptographic assumptions, these messages are computationally indistinguishable from random numbers to the adversary.
The simulator $\text{Sim}$ can mimic the messages that $A$ receives from $M(D)$ by generating them randomly. 
To simulate the outputs from the enclaves, $\text{Sim}$ interacts with the trusted party $T(D)$. 
Due to the integrity property of the enclaves in the real world, the $i$-th output from a real enclave is identical to the $i$-th output provided by the ideal functionality $T(D)$. 
Therefore, $\text{Sim}$ can simply forward the output from $T(D)$ to the simulated adversary.
$\text{Sim}$ must also simulate the adversary's ability to influence the protocol. 
The ideal functionality $T(D)$ is explicitly designed to allow $\text{Sim}$ to provide inputs on behalf of corrupted clients (e.g., malicious updates $g_j$ where $j \in C_{\text{cohort}}$) and to choose the completion order of rounds by setting the index $k$. 
This allows $\text{Sim}$ to perfectly replicate any malicious actions $A$ could take regarding corrupted client inputs and process scheduling. 
Due to the linearizability of the system, the inputs for a given round $k$ are independent of the outputs of any other concurrently executing but not-yet-completed round $i$ (where $i < k$). 
This non-blocking nature allows $\text{Sim}$ to simulate the adversary's reordering of round completions.
Finally, we consider the failure case, which occurs with probability $\delta_{privacy}$. In the ideal world, the trusted party $T(D)$ leaks the entire dataset $D$ to $\text{Sim}$. 
With this information, $\text{Sim}$ can perfectly simulate any possible view for the adversary, even one resulting from a catastrophic failure in the real protocol.

Since the views are computationally indistinguishable in both the success case (with probability $1 - \delta_\text{privacy}$) and the failure case (with probability $\delta_\text{privacy}$), we conclude that $\texttt{View}(A, M(D))$ is computationally indistinguishable from $\texttt{View}(\text{Sim}, T(D))$. This holds under the standard security assumptions for the underlying cryptographic primitives and the TEE's remote attestation mechanism.

\end{proof}

\begin{proposition}[Proposition~\ref{prop:availability}]
\label{prop:availability-appendix}
The probability $\delta_\text{interrupt}$ of interruption of auditing is 
$$
1-\left(1-\sum_{i=1}^{\tau} \binom{n\kappa\beta}{n_{\text{audit}}-\tau+i} \binom{n\kappa(1-\beta)}{\tau-i} \Big/ \binom{n\kappa}{n_{\text{audit}}}\right)^{n_\text{round}}.
$$
\end{proposition}

\begin{proof}
The method for computing the probability is almost identical to that in Theorem~\ref{theorem:linearizability_beta}. 
In each process, the interruption of auditing occurs when more than $n_\text{audit}-\tau$ auditors drop out. 
Since $n_\text{audit}$ auditors are selected from $n\kappa$ clients, which includes $n\kappa\beta$ dropout clients as assumed in Section~\ref{sec:threat_model}, the probability of this happening is given by $p=\Pr[a>n_\text{audit}-\tau]$, where $a$ is a random variable following a hypergeometric distribution with parameters $n\kappa$ and $n_\text{audit}$. 
Specifically, 
$$
p = \sum_{i=1}^{\tau} \binom{n\kappa\beta}{n_{\text{audit}}-\tau+i} \binom{n\kappa(1-\beta)}{\tau-i} \Big/ \binom{n\kappa}{n_{\text{audit}}}.
$$

This trial independently iterates across $n_\text{round}$ rounds.
In the same way as Theorem~\ref{theorem:linearizability_beta}, we get the probability as
$$
1-(1-p)^{n_\text{round}}.
$$

\end{proof}

\section{\revision{Formal Verification}}
\label{sec:formal_verification}

This appendix details the formal verification of the core security properties of our proposed planner enclave protocol. The goal is to formally prove its \textbf{linearizability} and integrity under a malicious server model (i.e., Theorem~\ref{theorem:linearizability}) while ensuring \textbf{liveness}. 
We use the symbolic verification tool \texttt{Tamarin Prover}  to model the protocol and verify its security against a powerful adversary. 
This formal proof provides strong evidence for the claims made in Theorem~\ref{theorem:linearizability} (linearizability) and Section~\ref{sec:availability-crash} (liveness).

\paragraph{Remark:}
To simplify the model, here, we represent auditor consensus as the signature of a single, representative auditor who correctly follows the protocol (i.e., $\gamma=0$). 
This abstraction is sufficient to demonstrate the linearizability and integrity, as long as at least one honest auditor is present. 
Even in the case where $\gamma>0$, Theorem~\ref{theorem:linearizability_beta} shows that by collecting $\tau$ signatures, it is highly probable that honest clients form a majority. 
This justifies modeling the behavior of the quorum as that of a single, representative honest auditor. 
Therefore, this simplification effectively demonstrates the essential security properties of our system.

\subsection{Modeling Framework: \texttt{Tamarin Prover}}

\texttt{Tamarin Prover} is a tool for the symbolic modeling and analysis of security protocols~\cite{meier2013tamarin}. It models protocols using \textit{Multiset Rewriting (MSR)} rules and assumes a Dolev-Yao adversary, which aligns with our threat model. 
Modeling enclave program logic with Tamarin has been shown to be an effective approach for verifying state continuity properties, as demonstrated in prior work~\cite{jangid2021towards}.
See the details of modeling enclave programs in Tamarin in~\cite{jangid2021towards}.

\subsection{Modeling Our Protocol in \texttt{Tamarin}}

We translated our protocol's core logic into a set of Tamarin MSR rules, modeling the interactions between the planner enclave, the clients (auditors), and the adversary.

\subsubsection{TEE Primitives}
\begin{itemize}
    \item \textbf{Planner Enclave and Remote Attestation:} An enclave is modeled by rules that use a unique hardware-derived attestation key (`!AttestationSecretKey`). 
    Remote attestation is modeled by having the enclave sign a message containing its identity (`'mrenclave'`), a fresh nonce, and other relevant data, which is called a quote. 
    The \textbf{integrity} of the enclave's execution is guaranteed by Tamarin's assumption that an entity faithfully follows its specified rules, which are verified by other parties through this quote.

    \item \textbf{Ecall/Ocall Flow:} The sequence of an Ecall,
    a subsequent Ocall to a client,
    and the final processing step inside the enclave is modeled using a \textbf{Linear Fact}.
    This fact can only be consumed once, ensuring that the enclave's internal state is securely passed from one step to the next in a sequential manner.

    \item \textbf{Confidentiality:} Sensitive data, such as a committed random number, is modeled as a fact 
    whose contents are inaccessible to the adversary. The adversary can see the fact exists but cannot learn the value of the input, thus modeling confidentiality.
\end{itemize}

\subsubsection{Protocol Logic}
\paragraph{Evidence Chain:}
An evidence chain is modeled as a sequence of \texttt{EvidenceChain(...)} facts, each of which has a hash of the previous block (\texttt{block\_hash}). 
Crucially, it contains not only the block hash but also a TEE-generated quote over that hash. 
In our model, this quote is represented as a signature by the enclave's attestation key (\texttt{\textasciitilde attestation\_sk}) over a tuple containing the \texttt{block\_hash}, the enclave's identity (\texttt{'mrenclave'}), and a nonce.
This design ensures the integrity and authenticity of the chain's progression. 
When a subsequent process begins, the enclave must verify this quote before it can proceed.
While the \texttt{EvidenceChain} facts themselves are passed through the public channel (making them subject to replay attacks by the adversary), the cryptographic link provided by the quote prevents the adversary from forging a new, invalid state or tampering with an existing one without being detected.

\paragraph{The Ecall/Ocall Flow}
\label{sec:appendix_ecall_ocall}

\begin{figure*}[t]
    \centering
    \includegraphics[width=\textwidth]{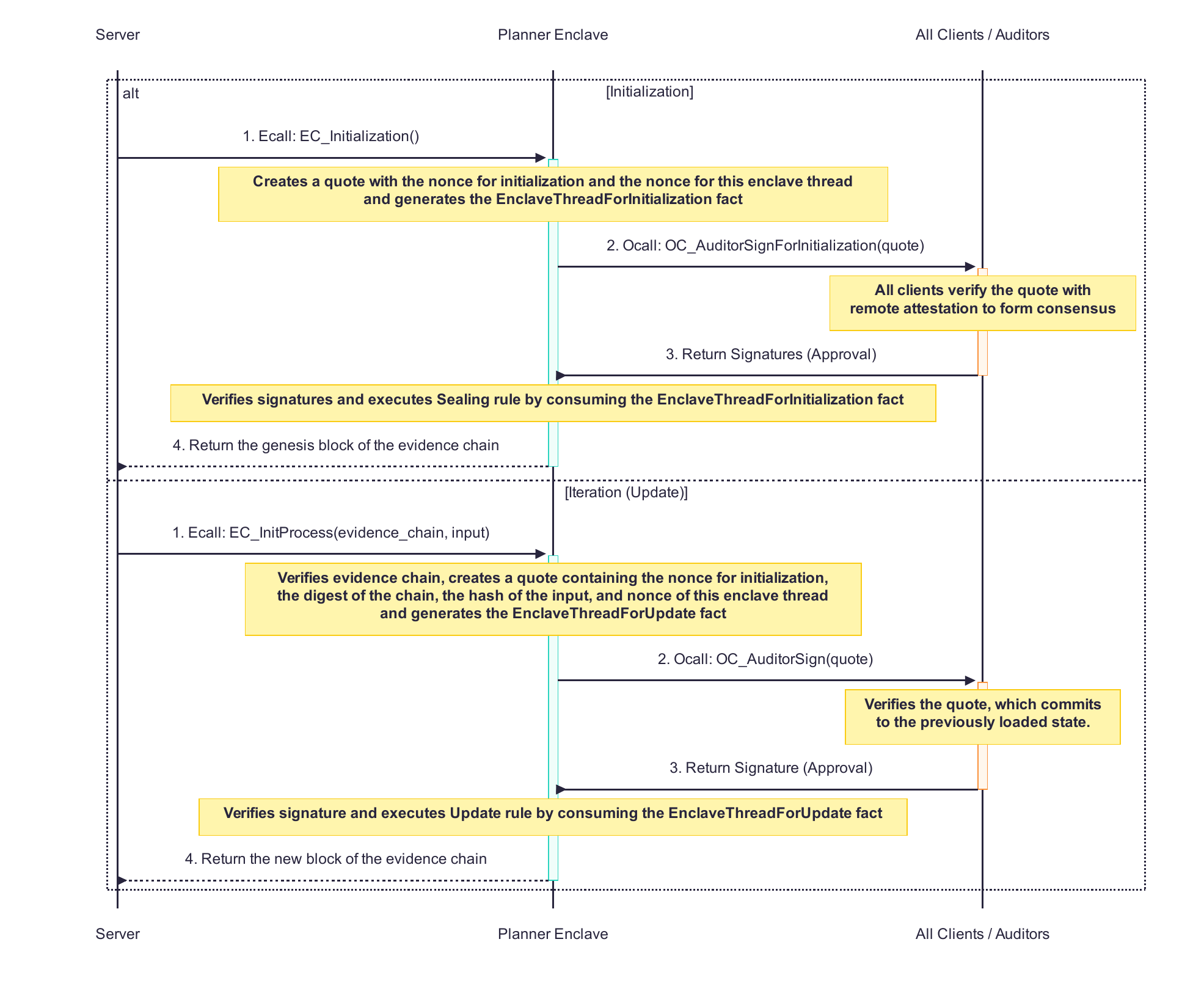}
    \caption{Sequence diagram of the state transition protocol between the Server, Planner Enclave, and Auditors.}
    \label{fig:formal_verif_flow}
\end{figure*}

The core of our model (i.e., Algorithm~\ref{alg:protocol}) lies in the Ecall/Ocall interactions between the planner enclave and the clients (auditors) using the evidence chain, as shown in Figure~\ref{fig:formal_verif_flow}. 
By modeling this flow in detail, we enable Tamarin to verify the linearizability of all state transitions. The following explains the key rules for the two main phases of the protocol.

\textbf{Initial State Generation (Genesis Block):}
The protocol begins by generating the first block of the evidence chain, which we refer to as the genesis block.
This corresponds to Line 2-7 of Algorithm~\ref{alg:protocol}.
This process is modeled by the following Tamarin rules:

\begin{itemize}
    \item \texttt{EC\_Initialization} (Ecall): This rule models the planner enclave initiating the creation of the evidence chain. The enclave generates a unique \texttt{nonce\_for\_init} for the entire chain and a \texttt{nonce\_for\_enclave\_thread} for this specific transaction. It then produces a quote---a signature using its hardware-derived \texttt{!AttestationSecretKey}---over its identity (\texttt{'mrenclave'}), the operation type (\texttt{'init'}), and these nonces. This quote, along with the nonces, is sent to the auditors via an Ocall (\texttt{Out(...)}). The linear fact \texttt{EnclaveThreadForInitialization} ensures state continuity within the enclave for the subsequent \texttt{Sealing} step.

    \item \texttt{OC\_AuditorSignForInitialization} (Ocall): An auditor handles the Ocall from \texttt{EC\_Initialization} with this rule. It first performs remote attestation by verifying the quote against the enclave's public key (\texttt{!AttestationPublicKey}). Upon successful verification, the auditor signs the \texttt{nonce\_for\_enclave\_thread} to signal its approval of the initialization and returns this signature to the enclave. This represents the initial consensus on the genesis block.

    \item \texttt{Sealing}: The enclave receives the auditor's signature and finalizes the genesis block by sealing it with the \texttt{EvidenceChain} fact. This fact includes the \texttt{block\_hash} of the genesis block, which is derived from the enclave's identity, the nonces, and the auditor's signature. The enclave then commits this state to the evidence chain, ensuring that it is now part of a verifiable history.
\end{itemize}

\textbf{State Update Process:}
Once initialized, the chain is extended through subsequent update processes.
This corresponds to Lines 11-20 of Algorithm~\ref{alg:protocol}.
This process is modeled by the following Tamarin rules:

\begin{itemize}
    \item \texttt{EC\_InitProcess} (Ecall): This rule is invoked for subsequent updates, modeling the planner enclave loading a previously committed state (\texttt{EvidenceChain(...)}). A critical step here is the verification of the \texttt{quoted\_block\_hash} from the loaded \texttt{EvidenceChain}, ensuring the enclave is starting from a legitimate, previously attested state. The enclave then prepares a new state transition based on a new \texttt{input} and generates a new quote. This new quote attests to the proposed update, binding the old \texttt{block\_hash} and the new \texttt{input} together. The new quote is then sent to the auditors via an Ocall, and the \texttt{EnclaveThreadForUpdate} linear fact maintains the enclave's state for the final \texttt{Update} rule.

    \item \texttt{OC\_AuditorSign} (Ocall): This rule models the core auditing function. An auditor receives the new quote from the \texttt{EC\_InitProcess} Ocall. It verifies this quote, confirming that the enclave is proposing a valid state transition. 
    By signing the quote, the auditor approves the update, as enforced by the \texttt{AuditorSignOnlyOnce} restriction (see Listing 1 below). This signature is the crucial element that allows the \texttt{Update} rule inside the enclave to finalize the new state and append it to the evidence chain.

    \item \texttt{Update}: Finally, the enclave processes the auditor's signature and updates the evidence chain with the new state. It generates a new \texttt{EvidenceChain} fact that includes the new \texttt{block\_hash}, which is derived from the previous block's hash, the new input. This new block is then appended to the evidence chain, completing the state transition.
\end{itemize}

This detailed modeling of the Ecall/Ocall flow allows Tamarin to verify that every state transition is correctly attested by the TEE and validated by an honest auditor. This provides the formal proof for the linearizability and integrity properties outlined in the main paper, demonstrating resilience against the rollback attacks modeled in \texttt{GenerateRollbackedEvidenceChain} (see Listing 2 below).

\subsubsection{Auditors and Adversary Model}

\paragraph{Honest Auditors:}

The honest behavior of auditors is enforced using a `restriction`. This restriction ensures that an auditor cannot approve two different histories branching from the same state. Specifically, an auditor can only sign once for any given \texttt{block\_hash}.
    \begin{lstlisting}[language=tamarin, caption={Restriction enforcing that an auditor signs only once per state.}, label={lst:auditor-restriction}]
restriction AuditorSignOnlyOnce:
    "
    All nonce_for_init block_hash input1 input2 #t1 #t2. 
    AuditorSignLabel(nonce_for_init, block_hash, input1) @t1
    & 
    AuditorSignLabel(nonce_for_init, block_hash, input2) @t2
    &
    not (input1 = input2)
    ==> 
    #t1 = #t2
    "
    \end{lstlisting}

\paragraph{Adversary:}

We explicitly model the adversary's primary attack vector: the rollback attack. The rule \texttt{GenerateRollbackedEvidenceChain} allows the adversary to take any previously observed \texttt{EvidenceChain} from the public channel and re-inject it into the system as a valid input for a new update process. This directly models the forking threat described in this paper.

\begin{lstlisting}[language=tamarin, caption={Rule explicitly modeling the adversary's rollback attack capability by replaying a past state.}, label={lst:rollback-attack}]
rule GenerateRollbackedEvidenceChain:
    [
        // Adversary takes a previously seen EvidenceChain from the network
        In(<..., block_hash, ..., quoted_block_hash>),
        // Verifies its quote to make it seem legitimate
        !AttestationPublicKey(attestation_pk),
        _restrict( verify(quoted_block_hash, ..., attestation_pk) = true )
    ]
    -->
    [
        // Re-creates the EvidenceChain to feed into the Ecall
        EvidenceChain(~nonce_for_init, ..., block_hash, ...)
    ]
\end{lstlisting}

\subsection{Security Properties and Verification Results}

We defined two key lemmas in Tamarin to represent our desired security properties and used the prover to check if they hold true against our adversary model.

\subsubsection{Proving Linearizability}
To prove that our system achieves linearizability, we defined the following lemma, which asserts that for any given state (identified by \texttt{block\_hash}), only one unique subsequent state can be committed.

\begin{lstlisting}[language=tamarin, caption={Lemma for proving linearizability.}, label={lst:lemma-linearizability}]
lemma UpdateOnlyOnceForAllBlockHash:
    "
    All nonce_for_init ... block_hash input1 input2 #t1 #t2.
    UpdateLabel(nonce_for_init, ..., block_hash, input1) @t1
    &
    UpdateLabel(nonce_for_init, ..., block_hash, input2) @t2
    &
    not (input1 = input2)
    ==> 
    #t1 = #t2
    "
\end{lstlisting}
This lemma states that it is impossible for two update operations with different inputs (\texttt{input1} $\neq$ \texttt{input2}) to successfully complete from the same state (\texttt{block\_hash}). This directly prevents the system from being forked into two different valid histories from a single point in time. An adversary attempting a rollback attack (modeled by \texttt{GenerateRollbackedEvidenceChain}) will fail to create a divergent history if a valid update has already been committed from that state.

\texttt{Tamarin Prover} successfully proved this lemma. This provides formal evidence that our protocol maintains a single, linearizable history and is resilient to rollback and forking attacks.

\subsubsection{Proving Liveness}
Next, we wanted to ensure that our protocol not only is secure but also guarantees \textbf{liveness}, allowing it to make progress even in the presence of crashes.

\begin{lstlisting}[language=tamarin, caption={Lemma for proving availability/executability.}, label={lst:lemma-availability}]
lemma UpdateAvailability: exists-trace
    "
    Ex nonce_for_init block_hash input #t1 #t2.
    UpdateLabel(..., block_hash, input) @t1
    &
    UpdateLabel(..., block_hash, input) @t2
    &
    #t1 < #t2
    "
\end{lstlisting}
The primary purpose of this lemma is to demonstrate the protocol's crash recovery capability. 
It proves that a valid execution trace exists, which implies that if an update process initiated at time \texttt{t1} were to fail (or “crash”) before completion, a subsequent process at a later time \texttt{t2} can successfully execute the \textit{exact same} update using the same input. State consistency is critical during this recovery. As proven by the \texttt{UpdateOnlyOnceForAllBlockHash} lemma, any attempt to recover the process with a \textit{different} input would be rejected. Therefore, this lemma, in conjunction with the linearizability proof, formally shows that our protocol supports consistent recovery from failures without compromising the integrity of the state.

\texttt{Tamarin Prover} successfully found a valid execution trace, proving the lemma. This confirms that our protocol has \textbf{liveness}.

\end{document}